\documentclass[twoside]{article}

\usepackage[final]{arxiv}
\usepackage{graphics}
\usepackage{graphicx}

\usepackage{amsmath,amsfonts,bm}

\def\eqref#1{equation~\ref{#1}}

\def\1{\bm{1}}

\def\vu{{\bm{u}}}

\def\vx{{\bm{x}}}
\def\vy{{\bm{y}}}
\def\vz{{\bm{z}}}

\DeclareMathAlphabet{\mathsfit}{\encodingdefault}{\sfdefault}{m}{sl}
\SetMathAlphabet{\mathsfit}{bold}{\encodingdefault}{\sfdefault}{bx}{n}

\DeclareMathOperator*{\argmax}{arg\,max}

\usepackage{hyperref}
\usepackage{url}
\usepackage{fontawesome}
\usepackage{multirow}
\usepackage{makecell}
\usepackage{xfrac}
\usepackage{siunitx,booktabs}
\usepackage{subcaption}

\usepackage{amsthm}
\usepackage{dsfont}
\theoremstyle{definition}
\newtheorem{definition}{Definition}[section]

\theoremstyle{remark}

\theoremstyle{lemma}
\newtheorem{lemma}{Lemma}[section]

\theoremstyle{proposition}
\newtheorem{proposition}{Proposition}[section]

\theoremstyle{theorem}
\newtheorem{theorem}{Theorem}[section]

\theoremstyle{corollary}
\newtheorem{corollary}{Corollary}[section]
\begin{document}

%

%
\runningauthor{KXY Technologies, Inc.}

\twocolumn[

\aistatstitle{LeanML: A Design Pattern To Slash Avoidable Wastes in Machine Learning Projects}

\aistatsauthor{Yves-Laurent Kom Samo}

\aistatsaddress{KXY Technologies, Inc. \\
\faEnvelope \, \texttt{yl@kxy.ai} \faMedium \,  \texttt{@Dr\_YLKS}  \faTwitter  \, \texttt{@Dr\_YLKS} \\
1762 Technology Dr. \\ San Jose, CA 95134, USA} ]

\begin{abstract}
We introduce the first application of the lean methodology to machine learning projects. Similar to lean startups and lean manufacturing, we argue that lean machine learning (LeanML) can drastically slash avoidable wastes in commercial machine learning projects, reduce the business risk in investing in machine learning capabilities and, in so doing, further democratize access to machine learning. The lean design pattern we propose in this paper is based on two realizations. First, it is possible to estimate the best performance one may achieve when predicting an outcome $\bm{y} \in \mathcal{Y}$ using a given set of explanatory variables $\bm{x} \in \mathcal{X}$, for a wide range of performance metrics, and without training any predictive model. Second, doing so is considerably easier, faster, and cheaper than learning the best predictive model. We derive formulae expressing the best $R^2$, MSE, classification accuracy and log-likelihood per observation achievable when using $\bm{x}$ to predict $\bm{y}$ as a function of the mutual information $I\left(\bm{y}; \bm{x}\right)$, and possibly a measure of the variability of $\bm{y}$ (e.g. its Shannon entropy in the case of classification accuracy, and its variance in the case regression MSE). We illustrate the efficacy of the LeanML design pattern on a wide range of regression and classification problems, synthetic and real-life.
\end{abstract}
\section{Introduction}
It is estimated that $25\%$ of commercial machine learning projects fail, and 9-in-10 fully trained predictive models are not good enough to make it to production. These wastes of resources are not without economic and ecological consequences.  Considering that 97\% of data still sits unused in organizations according to Gartner, Inc., the societal impact of this problem is bound to get worse if nothing is done. The approach consisting of devising processes to reduce unnecessary risk and slash wastes in business ventures,  often known as the lean methodology, has been applied to manufacturing and startups with great success.

In machine learning projects, wastes are typically of two kinds. The first kind are experiments that fail, and that we could have anticipated would fail. The second kind are experiments that fail, and that we let run till the end, even though we could have anticipated they would fail at some point during the execution. We argue that the exhaustive trial-and-error approach to building a new model from scratch or improving a production model, which is reinforced by AutoML platforms, contributes to wastes of the first kind in that they include several trials that are not worth running. Wastes of the second kind are typically due to waiting until a model is fully trained to realize that it is overfitted. 

To avoid these wastes, it suffices to answer a fundamental question prior to, and without, learning any predictive model: what are the theoretical-best performance metrics that may be achieved when using explanatory variables $\bm{x} \in \mathcal{X}$ to predict the business outcome $\bm{y} \in \mathcal{Y}$? The problem is a classification (resp. regression) problem when the set $\mathcal{Y}$ is finite (resp. continuous).  Avoiding experiments with low best outcome achievable would guard us from wastes of the first kind.  Additionally, wastes of the second kind can be mitigated by noting that, if during training a model's training performance far exceeds the theoretical-best achievable,  it would likely fail to generalize and,  as such, it should be preemptively and abruptly terminated. Related to the theoretical-best performances achievable is the mutual information, defined as $$I(\displaystyle \vy; \displaystyle \vx) := \int_{\mathcal{X} \times \mathcal{Y}} \log \frac{d P_{\displaystyle \vx, \displaystyle \vy}}{dP_{\displaystyle \vx} \otimes P_{\displaystyle \vy}} dP_{\displaystyle \vx, \displaystyle \vy}$$ where $P_{\displaystyle \vx, \displaystyle \vy}$ (resp. $P_{\displaystyle \vx}$, $P_{\displaystyle \vy}$) is the (joint) probability measure of $(\displaystyle \vx, \displaystyle \vy)$ (resp. $\displaystyle \vx$, $\displaystyle \vy$), and $d P_{\displaystyle \vx, \displaystyle \vy}/dP_{\displaystyle \vx} \otimes P_{\displaystyle \vy}$ is the Radon-Nikodym derivative of the joint probability measure with respect to the product measure of $P_{\displaystyle \vx} $ and $P_{\displaystyle \vy}$. \footnote{For special analytical expressions depending on whether variables are continuous, categorical or mixed see Table \ref{tab:mutual_information} in the Appendix.} 

The mutual information quantifies the extent to which $\bm{x}$ is informative about $\bm{y}$.  As such, one would expect that a mathematical relationship exists between mutual information and the highest performances achievable. As it turns out, this is indeed the case for the theoretical-best $R^2$,  Mean Square Error, classification accuracy,  and log-likelihood per observation achievable (without overfitting). We discuss these relationships in Section \ref{sct:theo_best}. We formally introduce the LeanML design pattern in Section \ref{sct:lean_ml} as the structure that predictive machine learning projects should adopt to slash avoidable wastes.  We showcase the feasibility and efficacy of the LeanML approach using synthetic and real-life experiments in Section \ref{sct:app}. But first we review related works.

\section{Related Works}
\label{sct:lit_rev}
Characterizing the theoretical-best classification accuracy achievable is a decades old problem. A prolific line of investigation has been to relate the conditional entropy $h\left(y \vert \bm{x}\right)$ to the error probability defined as $$e = \underset{\mathcal{M}: y \rightarrow \bm{x}  \rightarrow z}{\min} ~ \mathbb{P}\left(y \neq z \right) := 1 - \bar{\mathcal{A}}\left( P_{y, \bm{x}} \right),$$ where the min is taken across all $q$-classes classifiers $\mathcal{M}$ with generative graphical model $y \rightarrow \bm{x}  \rightarrow z$ (\cite{feder1994relations}).  

We use the following definition of the entropy\footnote{This definition includes the Shannon (when $\mu$ is the counting measure) and differential (when $\mu$ is Lebesgue's measure) entropies as special cases,  and extends these to vectors with a mix of continuous and categorical coordinates.} $$h\left(\displaystyle \vx \right) := -\int_{\mathcal{X}} \frac{d P_{\displaystyle \vx}}{d\mu}  \log  \frac{d P_{\displaystyle \vx}}{d\mu}  d\mu, $$ where $\mu$ is a base measure, from which the conditional entropy follows as $h\left(y \vert  \displaystyle \vx \right) := h\left(y,  \displaystyle \vx \right) - h\left( \displaystyle \vx \right)$. One such relation is Fano's strong bound (\cite{fano1949transmission}),  which reads $$\bar{\mathcal{A}}\left( P_{y, \bm{x}} \right) \leq \bar{h}_q^{-1}\left( h\left(y \vert \bm{x}\right) \right),$$ where $\bar{\mathcal{A}}\left( P_{y, \bm{x}} \right)$ is the highest classification accuracy achievable, and $\bar{h}_q^{-1}$ is the inverse of the function $$\bar{h}_q(a) := -a\log a - (1-a) \log \frac{1-a}{q-1}, ~~ a \in [ \frac{1}{q}, 1].$$ Along the same line, \cite{hellman1970probability} proved that $$1 - \frac{h(y\vert \bm{x})}{2 \log 2} \leq \bar{\mathcal{A}}\left( P_{y, \displaystyle \vx} \right),$$ where the logarithm is natural and the entropy is in nats, as it will be the case throughout this paper.  Although the Fano and Hellman-Raviv bounds can be very far apart,\footnote{In the binary case, the gap can be as wide as $0.16$.} it has been shown that they are both tight (\cite{zhao2013beyond}). In other words,  for a given value of the conditional entropy $h(y\vert \bm{x})$,  the highest classification accuracy we may achieve (i.e. $\bar{\mathcal{A}}\left( P_{y, \bm{x}} \right)$) can either be $\bar{h}_q^{-1}\left( h(y\vert \bm{x})\right)$,  $1 - \frac{ h(y\vert \bm{x})}{2\log 2}$,  or anywhere in between,  depending on the nature of the joint distribution $P_{y, \displaystyle \vx}$. 

In Section \ref{sct:fano}, we provide a constructive proof of Fano's inequality, thanks to which we may conclude that Fano's strong bound can always be reached so long as explanatory variables are uniformly informative about the label --- i.e. the function $* \to h\left(y \vert \bm{x}=*\right)$ is constant on the input domain $\mathcal{X}$ (see Theorem \ref{theo:hacc}).  This sufficient condition is far from necessary however and, in practice,  we find that variations of $* \to h\left(y \vert \bm{x}=*\right)$ on the input domain that are able to push $\bar{\mathcal{A}}\left( P_{y, \bm{x}} \right)$ to the Hellman-Raviv bound are pathological in nature. Even when the strong Fano bound is not reached, it can be used as upper-bound of $\bar{\mathcal{A}}\left( P_{y, \bm{x}} \right)$ to mitigate wastes of the first and second kinds in machine learning projects.  In such an instance, the closer $\bar{\mathcal{A}}\left( P_{y, \bm{x}} \right)$ is to Fano's strong bound, the more wastes we will be able to anticipate and avoid.

As for the true log-likelihood per observation of a supervised learner $\mathcal{M}$ with predictive pdf or pmf $p_\mathcal{M}$,  defined as $$\mathcal{LL}\left( \mathcal{M} \right) := E_{P_{y, \bm{x}}}\left[ \log p_\mathcal{M} \right],$$ it follows from \cite{reid2011information} and \cite{duchi2018multiclass} that, in the case of binary and multiclass classification problems, the highest true log-likelihood per observation is equal to the negative conditional entropy $-h\left(y \vert \bm{x}\right)$. Proposition \ref{prop:ll} extends this result to regression problems.

In regards to the regression Mean Square Error (MSE),  when $y$ and $\displaystyle \vx$ are $L^2$,  minimizing the MSE incurred when predicting $y$ using $\displaystyle \vx$ is equivalent to finding the orthogonal projection of $y$ on the sigma-algebra generated by $\displaystyle \vx$.  The solution is widely known to be the conditional expectation $E\left(y \vert \displaystyle \vx\right)$ (\cite{dellacherie2011probabilities}),  and the associated optimal MSE is $$\bar{MSE}_c\left( P_{y, \bm{x}} \right) :=E\left[y^2 - E\left(y \vert \displaystyle \vx\right)^2 \right]. $$ \cite{brillinger2004some} suggested using the inequality $$E\left[ \left(y- f(\displaystyle \vx) \right)^2 \right] \geq  \frac{e^{2h(y)}}{2\pi e} e^{-2I\left(y; \displaystyle \vx \right)}$$ to lower-bound the MSE that one may achieve when using $\displaystyle \vx$ to predict $y$.  However,  this lower-bound is not tight in the non-Gaussian case,  and the MSE cannot always be as small as $\frac{e^{2h(y)}}{2\pi e} e^{-2I\left(y; \displaystyle \vx \right)}$. 

More generally,  it is not possible to directly estimate the optimal MSE $\bar{MSE}_c\left( P_{y, \bm{x}} \right)$, without first learning the best predictive model $f : \bm{x} \to E\left(y \vert \displaystyle \vx\right)$,  or making arbitrary distribution assumptions such as assuming Gaussianity, which would be contrary to the objective and the spirit of LeanML.  

Fortunately, in Section \ref{sct:trick}, we introduce a simple information-theoretical trick which we denote the \emph{variance-entropy swap trick},  to generalize performance or loss metrics such as the $R^2$ and the MSE, that are defined using a conditional variance term, to non-Gaussian distributions. The generalized metrics are identical to the classical ones in the Gaussian case (e.g. Ordinary Least Square and Gaussian Process Regression (\cite{rasmussen2003gaussian})), but better capture the notion of risk for fat-tailed residual distributions. Although classic and generalized metrics can vary drastically for a given model $\mathcal{M}$, we show empirically that their theoretical-best values are so close that one may be used as proxy for the other.  Given that estimating the theoretical-best generalized metrics can be done without making arbitrary distribution assumptions and without learning any predictive model, this allows us to circumvent the aforementioned limitation in estimating theoretical-best classic metrics.

Coincidentally, the generalized $R^2$ we introduce,  namely $$R^2\left( \mathcal{M}\right) := 1-e^{-2I\left(\bm{y}; \bm{z} \right)}$$ when model $\mathcal{M}$ makes prediction $\bm{z} = f\left(\bm{x}\right)$ about $\bm{y}$, naturally extends to classification problems. The idea of applying the variance-entropy swap trick to extend the $R^2$ to classification problems is closely related to the pseudo-$R^2$ introduced by  \cite{cox1989analysis} for logistic regressions,  namely $$\text{Pseudo-R}^2\left( \mathcal{M} \right) = 1 - e^{-2\left( \hat{\mathcal{LL}}\left( \mathcal{M} \right) - \hat{\mathcal{LL}}\left( \mathcal{M}^0 \right) \right)},$$ where $\hat{\mathcal{LL}}$ is the empirical log-likelihood per observation, and $\mathcal{M}^0$ is the baseline model consisting of ignoring explanatory variables. In effect, $E\left[ \hat{\mathcal{LL}}\left( \mathcal{M} \right) - \hat{\mathcal{LL}}\left( \mathcal{M}^0 \right)  \right] = I\left(\bm{y}; \bm{z} \right).$

\cite{joe1989estimation,joe1989relative} also suggested using $1-e^{-2I\left(\bm{y}; \bm{z} \right)}$, but as a generalized correlation coefficient between $\bm{y}$ and $\bm{z}$. We derive the highest generalized $R^2$ achievable and the lowest MSE achievable in Proposition \ref{prop:best}.

\section{Theoretical-Best Supervised Learning Performances}
\label{sct:theo_best}
We consider predicting an output $\displaystyle \vy \in \mathcal{Y}$ using inputs $\displaystyle \vx \in \mathcal{X}$. The problem is a classification (resp. regression) problem when $\displaystyle \vy$ is categorical (resp. continuous). We use $\mathcal{M}$ to denote a generic supervised learning model which, without loss of generality, we represent by the generative graphical model $\displaystyle \vy \rightarrow \displaystyle \vx  \rightarrow \displaystyle \vz $. $\displaystyle \vz \in \mathcal{Y}$ typically represents the knowledge the model extracts about $\displaystyle \vy$ from  $\displaystyle \vx$. $\mathcal{M}^\infty$ denotes the \emph{oracle} supervised learner defined by the generative graphical model $\displaystyle \vy \rightarrow  \displaystyle \vx \rightarrow \displaystyle \vz^\infty$ where $P_{\displaystyle \vy \vert \displaystyle \vz^\infty} = P_{\displaystyle \vy \vert \displaystyle \vx}$. In other words, $\displaystyle \vz^\infty$ still has all the insights about $\displaystyle \vy$ that were in $\displaystyle \vx$. $\mathcal{M}^0$ denotes the \emph{baseline} (unbiased) supervised learner defined by the generative graphical model $\displaystyle \vy \rightarrow \displaystyle \vx \rightarrow \displaystyle \vz^0$ where $P_{\displaystyle \vy \vert \displaystyle \vz^0} = P_{\displaystyle \vy}$ (that is, $\displaystyle \vz^0$ has no insights about $\displaystyle \vy$) and $E(\displaystyle \vy)=E(\displaystyle \vz^0)$ for regression problems. As previously mentioned, we use the symbols $y$ and $z$ in-lieu-of $\displaystyle \vy$ and $\displaystyle \vz$ when the treatment is specific to one-dimensional outputs.

\subsection{The Variance-Entropy Swap Trick}
\label{sct:trick} 
It is well known that variance and conditional variance are weak measures of risk and residual risk for most distributions.  Gaussian distributions are a notable exception. The variance (resp.  conditional variance) of a Gaussian is as good a measure of uncertainty (resp.  conditional uncertainty) as it gets in the sense that any other measure of uncertainty (resp. conditional uncertainty) can be expressed as a function thereof.  The entropy and the conditional entropy are much better alternatives.  Many distributions such as the Cauchy distribution have undefined or infinite moments, but well-defined and finite entropies.  Additionally,  two random variables are independent if and only if entropy and conditional entropy are equal, but variance and conditional variance do not suffice to conclude statistical independence.

In regression problems,  loss functions and performance metrics that are based on a conditional variance implicitly rely on the assumption that residuals are Gaussian to be general enough measures of residual risk.  For instance,  for a regression model $\mathcal{M}$ making prediction $z = f\left(\displaystyle \vx \right)$ associated to inputs $\displaystyle \vx$,  the (classic) Mean Square Error,  which we recall is defined as $$MSE_c \left(\mathcal{M}\right) := E\left[\left(y-z\right)^2 \right] =  \mathbb{V}\text{ar} \left( y \vert z \right) + \left[E\left( y-z \right)\right]^2,$$
is often used as loss function in conjunction with the Gaussian assumption on residuals (e.g. in GP regression and OLS). 

Similarly,  the (classic) $R^2$ defined as $$R_c^2 \left(\mathcal{M}\right)  = 1 - \frac{\mathbb{V}\text{ar}\left(y \vert z \right)}{\mathbb{V}\text{ar}(y)}, $$
is only a general enough measure of regression performance when $y$ is Gaussian both unconditionally,  and conditional on $z$,  an assumption often embedded in regression models (e.g.  OLS and GP regression),  which we will refer to from now on as the Gaussian assumption.

When the Gaussian assumption is met,  we have $\mathbb{V}\text{ar}\left(y \vert z \right)/\mathbb{V}\text{ar}(y) = e^{-2I\left( y; z \right)}$, and we may simply rewrite the (classic) MSE and $R^2$ as 
\begin{align}
\label{eq:new_mse}
MSE_c \left(\mathcal{M}\right) = \mathbb{V}\text{ar} \left( y \right) e^{-2I\left(y  ; z \right)} + \left[E\left( y  -  z \right)\right]^2
\end{align}
and
\begin{align}
\label{eq:new_rsq}
R_c^2 \left(\mathcal{M}\right) = 1 - e^{-2I\left( y; z \right)}.
\end{align}
When the Gaussian assumption is not met, $MSE_c\left(\mathcal{M}\right)$ and $R_c^2\left(\mathcal{M}\right)$ do not fully capture residual risk. Instead,  we use Equations (\ref{eq:new_mse}) and (\ref{eq:new_rsq}) as more general and robust alternatives.  

\begin{definition}
The \emph{generalized Mean Square Error} of a regression model $\mathcal{M}$ with generative graphical model $\displaystyle \vy \rightarrow \displaystyle \vx  \rightarrow \displaystyle \vz $ reads
\begin{align}
MSE \left(\mathcal{M}\right) = \mathbb{V}\text{ar} \left( \displaystyle \vy \right) e^{-2I\left(\displaystyle \vy  ; \displaystyle \vz \right)} + \left[E\left( \displaystyle \vy  - \displaystyle \vz \right)\right]^2.
\end{align}
\end{definition}

\begin{definition}
The \emph{generalized} $R^2$ of a supervised learner $\mathcal{M}$ with generative graphical model $\displaystyle \vy \rightarrow \displaystyle \vx  \rightarrow \displaystyle \vz $ reads
\begin{align}
\label{eq:new_rsq2}
R^2 \left(\mathcal{M}\right) = 1 - e^{-2I\left( \displaystyle \vy; \displaystyle \vz \right)}.
\end{align}
\end{definition}

We refer to swapping the ratio $\mathbb{V}\text{ar}\left(\displaystyle \vy \vert \displaystyle \vz \right)/\mathbb{V}\text{ar}(\displaystyle \vy)$ for $e^{-2I\left( \displaystyle \vy; \displaystyle \vz \right)}$ as the \emph{variance-entropy swap trick}.  

\textbf{Remarks}: Equation (\ref{eq:new_rsq2}) extends the notion of $R^2$ to classification problems.  Unlike in regression problems,  the perfect generalized $R^2$ in a $q$-classes classification problem is not $1$ but $1-e^{-2\log q}$.  This is an artifact of the difference between differential and Shannon mutual informations of two fully dependent random variables.  

As previously discussed, when both $z$ and $\epsilon := y-z$ are Gaussian,  $R^2\left(\mathcal{M}\right) = R_c^2\left(\mathcal{M}\right)$ and  $MSE\left(\mathcal{M}\right) = MSE_c\left(\mathcal{M}\right)$. More generally, when either $z$ or $\epsilon$ is Gaussian (e.g. GP regression with a non-Gaussian noise, or Deep Regession with Gaussian residuals), it is easy to prove that $R^2\left(\mathcal{M}\right) \leq R_c^2\left(\mathcal{M}\right)$ and  $MSE\left(\mathcal{M}\right) \geq MSE_c\left(\mathcal{M}\right)$, and that the gap grows with the entropy deficit of the non-Gaussian variable out of the two (relative to the entropy of the Gaussian distribution with the same variance). 

One way to think about this is that, for regression problems, generalized metrics penalize classic metrics for failing to account for risk beyond the second order.

\subsection{Maximum Achievable True Log-Likehood Per Observation}
The following result is a direct consequence of the non-negativity of the KL divergence, and is proved in Appendix \ref{proof:prop:ll}.
\begin{proposition}
\label{prop:ll}
The highest true log-likelihood per observation (defined as $\mathcal{LL}\left( \mathcal{M} \right) := E_{P_{y, \bm{x}}}\left[ \log p_\mathcal{M} \right]$) achievable by a supervised learner $\mathcal{M}$ using $\bm{x}$ to predict $\bm{y}$ and that has predictive pmf or pdf $p_\mathcal{M}$, is 
\begin{align}
\bar{\mathcal{LL}}\left( P_{\displaystyle \vy, \displaystyle \vx} \right) :&=\mathcal{LL}\left(\mathcal{M}^0\right) +  I\left(\displaystyle \vy ; \displaystyle \vx \right) \nonumber \\
&=  -h\left(\displaystyle \vy\right) + I\left(\displaystyle \vy ; \displaystyle \vx \right). \nonumber
\end{align}
It is achieved by the \emph{oracle} supervised learner $\mathcal{M}^\infty$.
\end{proposition}

\subsection{Maximum Achievable $R^2$ and Minimum Achievable MSE}
The following result is a direct consequence of the data processing inequality (\cite{incover1999elements}), and is proved in Appendix \ref{proof:prop:best}.

\begin{proposition}
\label{prop:best}
The highest generalized $R^2$ and lowest generalized MSE achievable by a supervised learner using $\displaystyle \vx$ to predict $\displaystyle \vy$ read $$\bar{R^2}\left( P_{\displaystyle \vy, \displaystyle \vx} \right) := 1- e^{-2 I\left(\displaystyle \vy; \displaystyle \vx \right)}$$ and 
\begin{align}
\bar{MSE}\left( P_{y, \displaystyle \vx} \right) :&= e^{-2I\left(y ; \displaystyle \vx \right)} \mathbb{V}\text{ar}\left( y \right) \nonumber \\
&= e^{-2I\left(y ; \displaystyle \vx \right)}MSE\left(\mathcal{M}^0\right). \nonumber
\end{align}
They are both achieved by the \emph{oracle} supervised learner $\mathcal{M}^\infty$.
\end{proposition}
\textbf{Remarks}: Although the gap between generalized and classic performance metrics can be fairly large depending on the model $\mathcal{M}$,  in our experience (including the experiments of Section \ref{sct:app}), the gap between the theoretical-best classic metrics and the theoretical-best generalized metrics, which only depends on the true distribution $P_{y, \bm{x}}$, is typically far smaller, to the point of justifying using an estimation of a theoretical-best generalized $R^2$ (resp. MSE) as a proxy for the theoretical-best classic MSE (or $R^2$).  

We stress once more that, unless we make an arbitrary assumption on the true generative distribution such as the Gaussian assumption, the theoretical-best classic $R^2$ (resp. MSE) \emph{cannot be estimated directly} without first learning the best predictive model $\bm{x} \to E\left(y \vert \bm{x} \right)$,  which would defeat the purpose of the LeanML paradigm.
\subsection{Maximum Achievable Classification Accuracy}
\label{sct:fano}
In a $q$-classes classification problem, without loss of generality, we assume that the set of classes is $\mathcal{Y} = \{1, \dots, q\}$. The following result, which we derive in Appendix \ref{sct:fano_deriv}, provides specific and practical conditions under which Fano's strong bound (\cite{fano1949transmission}) is reachable.
\begin{theorem}
\label{theo:hacc}
The highest accuracy $\bar{\mathcal{A}}(P_{y, \displaystyle \vx})$ achievable by a classifier using $\displaystyle \vx$ to predict a categorical random variable $y \in \{1, \dots, q\}$ satisfies the strong Fano inequality $$\bar{\mathcal{A}}(P_{y, \displaystyle \vx}) \leq \bar{h}^{-1}_q\left( h(y)  - I\left( y; \displaystyle \vx\right) \right).$$Additionally, $$\bar{\mathcal{A}}(P_{y, \displaystyle \vx}) = \bar{h}^{-1}_q\left( h(y)  - I\left( y; \displaystyle \vx\right) \right)$$ and the \emph{oracle} classifier $\mathcal{M}^\infty$ achieves $\bar{\mathcal{A}}(P_{y, \displaystyle \vx})$, when the entropy of the conditional distribution, namely $h\left( y \vert \displaystyle \vx = * \right)$, is the same for all values $*$ of $\displaystyle \vx$ (i.e. $\displaystyle \vx$ is no more informative about $y$ in certain parts of the domain $\mathcal{X}$ than others), and when $q=2$ or the $(q-1)$ least likely outcomes under the conditional distribution $P_{y \vert \displaystyle \vx}$ are always equally likely (i.e. the information in $\displaystyle \vx$ about $y$ leaves no room for a clear runner-up).
\end{theorem}
\textbf{Remarks:} In multiclass classification problems (i.e. $q>2$),  when the `no clear runner-up' condition of Theorem \ref{theo:hacc} is not met,  to reach the strong Fano bound,  we can trade the question `how accurate can a classifier using $\displaystyle \vx$ to predict $y$ be overall' for the (arguably more granular) $q$ questions `how accurate can a classifier using $\displaystyle \vx$ to predict whether $y$ will take the specific value $i$ be' (i.e. $i$-vs-rest classification) for each $i \in \{1, \dots, q\}$.  The latter are binary classification problems to which the `no clear runner-up' condition does not apply.  

As for the uniform-informativeness condition,  it is a sufficient condition for the bound to be reachable, but it is far from being necessary.  In our experience,  the effect of any departure from this condition will typically be small relative to the estimation error of the mutual information, and variations of $* \to h\left(y \vert \bm{x}=*\right)$ on the input domain that are able to push $\bar{\mathcal{A}}\left( P_{y, \bm{x}} \right)$ to the Hellman-Raviv bound are pathological in nature.

\section{Making Machine Learning Lean}
\label{sct:lean_ml}
To slash avoidable wastes in supervised learning projects,  we propose structuring them in a manner that abides by two core principles.

\subsection{The LeanML Principles}
\textbf{Principle \#1: Always condition running an experiment on its feasibility.}

Whether a data scientist is trying to predict a specific business outcome for the first time, or trying to improve an already deployed production model, it is crucial that he/she first estimates the best outcome he/she should realistically hope for, before starting the project. If a satisfactory enough outcome cannot be generated, then starting the project would be wasteful.

For instance, prior to training a predictive model, a data scientist should always first value his/her data (i.e. estimate the highest performance achievable). If the theoretical-best performance achievable is not satisfactory for the business use case, he/she should focus on gathering additional and complementary explanatory variables,  value the new set of explanatory variables, and repeat until he/she gathers explanatory variables from which a desirable business outcome can be achieved.

Similarly, a data scientist attempting to improve a deployed production model should first question the extent to which it is possible to do so. Because the data scientist stumbled upon a fancy new class of models he wasn't aware of, doesn't mean his/her production model can be improved. To determine by how much the production model can be improved in a model-driven fashion (i.e. using the same explanatory variables), the data scientist should compare the performance of the production model to the best performance achievable. Only if there is a large enough gap, should the data scientist consider training new models.  

If the data scientist finds that the production model is performing at the theoretical best level, then he/she should be looking for additional and complementary explanatory variables to use in order to boost performance. Once new explanatory variables are found that the data scientist suspects have the potential to boost the performance of the production model, the data scientist should first compute the highest performance boost he/she should expect.\footnote{By comparing the highest performances achievable using the old set of explanatory variables to the highest performance achievable using the old and new set of explanatory variables combined.} Only if the expected performance boost is large enough, should the data scientist attempt to improve the production model by retraining models in his/her toolbox with the new set of explanatory variables.

\textbf{Principle \#2: Pro-actively terminate an experiment you started, as soon as you can reliably determine it will fail.}

Another big source of wastes in ML projects is the need to discard overfitted models.  To detect when a model being trained is likely to overfit, we can compare the running lowest loss (e.g. log-likelihood per observation or MSE) or the running highest performance (e.g. $R^2$ or classification accuracy)  to the theoretical-best achievable.  If the running loss (resp. performance) is lower (resp. higher) than the theoretical-best by more than a (possibly null) threshold,  then this is a strong indication that the fully trained model will end up overfitting,  and therefore that we need to `cut our losses' by preemptively terminating training. The early-termination we advocate here is not to be confused with `early-stopping' methods that aim at preventing overfitting by stopping an optimizer before it has a chance to overfit (\cite{smale2007learning, yao2007early}); it complements these methods. Indeed, whether `early-stopping' methods are utilized or not, if the running loss (resp. performance) happens to be much lower (resp. much higher) than the theoretical-best during training, then this is strong evidence that the model being trained will end up overfitting, and that any resource spent between when this determination is made and when training stops would go to waste.

\subsection{The LeanML Design Pattern}
The LeanML design pattern is an implementation of the foregoing LeanML principles, and advocates structuring predictive modeling projects as follows.

\textbf{Step 1: Data Valuation.} The highest performance achievable using available explanatory variables $\bm{x}$ to predict the business outcome of interest $\bm{y}$ should be estimated, and the project should not proceed until explanatory variables are found that could yield a satisfactory outcome when used to predict the business outcome.

\textbf{Step 2: Model-Free Variable Selection.} Variables or features that are either not informative about the label $\bm{y}$ or redundant should be eliminated based on the highest performances achievable. Failure to properly select variables or features could result in lengthier and costlier training, a higher chance of overfitting, and more rapid performance decay when the model is used live. Additionally, the more features a model uses, the more susceptible real-time instances of the model will be to an outage of the feature delivery service(s), with obvious impact on the bottom line, not least higher maintenance costs. An example implementation is the greedy model-free variable selection algorithm that proceeds as follows. The first variable is selected as the variable that could yield the highest performance when used by itself. For $i>1$, the $i$-th variable is selected as the variable, among all variables not yet selected that, when added to the $i-1$ variables previously selected, will increase the highest performance achievable the most. The selection stops when a reasonable criteria is met, such as the number of variables selected so far exceeding a capacity threshold and/or the highest performance achievable with selected variables exceeding a certain percentage (e.g. 95\%) of the highest performance achievable using all variables.

\textbf{Step 3: Lean Model Building.} Model training should be terminated as soon as the running loss (resp. performance) is lower (resp. higher) than the theoretical-best estimated in Step 1, by more than a (possibly null) threshold on the basis that this is strong indication that the model will end up overfitting. Terminated models shoud be discarded.

\textbf{Step 4: Lean Model Improvement.} Before attempting to improve a model, data scientists should first assess the extent to which it can be improved.  A model whose performance is close to the theoretical-best performance estimated in Step 1 cannot be improved upon without resorting to additional and complementary explanatory variables. When the model $\mathcal{M}_0$ to improve, which we assume makes prediction $f_0(\bm{x})$ associated to $\bm{x}$, does not perform at the theoretical-best level, comparing the outputs of the model-free variable selection in Step 2 applied to the two pairs $\left(\bm{y}, \bm{x} \right)$ and $\left(f_0(\bm{x}), \bm{x} \right)$ can help shed some light on variables the model $\mathcal{M}_0$ under-utilized.  For regression problems, we may go further and adopt an iterative approach by repeating Steps 1-3, this time applied to regression residuals $\bm{y}_1 = \bm{y}-f_0(\bm{x})$,  to arrive at model $\mathcal{M}_1$ with prediction $f_1(\bm{x})$ about residual $\bm{y}_1$.  Done $i+1$ times,  this leads to the fine-tuned additive model $\mathcal{M}$ making predictions $f(\bm{x}) = \sum_{k=0}^i f_k\left(\bm{x} \right)$ about $\bm{y}$. It is important to note that, at each iteration,  Step 2 would effectively only select variables whose dependencies to the output $\bm{y}$ still aren't properly accounted for by the running additive model. Similarly, before attempting to improve model $\mathcal{M}_0$ trained with $\bm{x}$ using new explanatory variables $\bm{x}^\prime$, it is important to first estimate how much incremental performance $\bm{x}^\prime$ can bring about by comparing the highest performance achievable when predicting $\bm{y}$ using $\bm{x}$ and using $[\bm{x},  \bm{x}^\prime]$. Unless $\bm{x}^\prime$ can bring about a high enough performance increase, it wouldn't be worth retraining candidate models using $[\bm{x},  \bm{x}^\prime]$.

\section{Experiments}
\label{sct:app}
We estimate mutual informations using the recent MIND estimator of \cite{pmlr-v130-kom-samo21a}, which we find particularly suitable for LeanML, as it is very data-efficient and copes well with large input dimensions. See Appendix \ref{sct:mi_discussion} for an extended discussion on mutual information estimation, where we provide new insights into the links between MIND, MINE (\cite{belghazi2018mutual}) and NWJ (\cite{nguyen2010estimating}) so as to illustrate how exactly MIND is able to be much more data-efficient than competing alternatives. To estimate the differential entropy $h\left( \displaystyle \vy \right)$ of a random vector $\displaystyle \vy = \left(y_1, \dots, y_d \right)$, we suggest using the entropy decomposition $h\left( \displaystyle \vy \right) = h\left(  \displaystyle \vu_y \right) + \sum_{i=1}^d h(y_i)$ where $h\left(  \displaystyle \vu_y \right)$ is the entropy of the copula of $\displaystyle \vy$. We find that one-dimensional differential entropies $h\left(y_i\right)$ are best estimated using M-estimators coupled with kernel density estimation (\cite{parzen1962estimation}) or Dirichlet Process mixture models (\cite{escobar1995bayesian, teh2005sharing}).  As for estimating copula entropies,  this is only needed to estimate $\bar{\mathcal{LL}}$  when $\bm{y}$ is multi-dimensional, and we also suggest using MIND. The variance term in $\bar{MSE}$ is estimated as sample variance, and the Shannon entropy in $\bar{\mathcal{A}}$ is estimated using the frequency based plug-in estimator. 

\textbf{Data Valuation Experiments}: We illustrate the accuracy of our data valuation approach using synthetic data of which we may calculate the ground truth. We use $\mathcal{X} = [0, 1]^d$ and we choose as $P_{\bm{x}}$ the $d-$dimensional standard uniform. For regression problems,  given a function $f$, we define $y = f\left( \bm{x} \right) + \epsilon$, where $\epsilon$ is an independent Gaussian noise with standard deviation $\sigma$. For classification problems, we define $y = (1-s) \mathds{1}\left[ f\left(\bm{x}\right) \geq m \right] + s \mathds{1}\left[ f\left( \bm{x} \right) < m \right]$, where $s$ is an independent Bernoulli random variable taking value $1$ with probability $p_e$, and $0$ otherwise, and $m=E\left(f(\bm{x})\right)$. We use the following $4$ functions: $f_1(\bm{x}) \propto \sum_{i=1}^d \frac{x_i}{i}$,  $f_2(\bm{x}) \propto \sqrt{\left\vert \sum_{i=1}^d \frac{x_i}{i} \right\vert}$,  $f_3(\bm{x}) \propto -\left(\sum_{i=1}^d \frac{\vert x_i-0.5 \vert}{i}\right)^3$,  and $f_4(\bm{x}) \propto \tanh \left(\frac{5}{2} \sum_{i=1}^d \frac{(x_i-0.5)^2}{i} \right)$, with $\bm{x} := \left( x_1, \dots, x_d\right)$. The scaling coefficient of each function is chosen so that the sample variance is $1$. In regression problems, the highest achievable classic $R^2$ is easily found to be $\frac{1}{ 1+\sigma^2}$ and the lowest classic MSE achievable is easily found to be $\sigma^2$.  For classification problems,  regardless of $f$, when $p_e=0$, $s$ is always $0$ and $y=\mathds{1}\left[ f\left(\bm{x}\right) \geq m \right] := z$ can be perfectly classified from $\bm{x}$. The effect of $s$ for $p_e > 0$, is to switch the value of $z$ (from $0$ to $1$ and vice-versa) with probability $p_e$. Thus, the highest achievable accuracy should always be $\bar{\mathcal{A}}\left(P_{y, \bm{x}} \right)=1-p_e$. Because every $z$ has the same probability of being switched for any $\bm{x}$, the uniform-informativeness condition of Theorem \ref{theo:hacc} is met and, given that $q=2$,  Fano's strong bound can be reached.  We use every combination of $d \in \{1, 2, 5, 10\}$ and $\bar{R^2}\left( P_{y, \bm{x}} \right) \in \{0.99,  0.75, 0.5, 0.25\}$ for regression problems and $\bar{\mathcal{A}}\left(P_{y, \bm{x}} \right) \in \{1,  0.99,  0.75, 0.5\}$ for classification problems.  To gauge the variability of our estimators, for each combination, we run $10$ independent experiments, each with its own set of noise observations $\epsilon$ or $s$, but all with the same input draws, and we report the mean and the standard deviation of estimated theoretical-best performances across the $10$ runs. Each experiment is based on $d*1000$ i.i.d.  samples, and we estimate $m$ using simple Monte Carlo. Results are partly illustrated in Figures \ref{fig:best_rmse} and \ref{fig:best_rsq} for $d=1$ and $d=2$,  and fully summarized in Table \ref{tab:summary} for $d=10$.  All individual results are reported in Tables \ref{tab:full_rsq}, \ref{tab:full_rmse} and \ref{tab:full_acc} in the Appendix.  Although the Gaussian assumption is not met in these regression experiments, it can be seen in Table \ref{tab:summary} that our estimation of the theoretical best (generalized) metrics is able to recover the true theoretical best (classic) metrics almost perfectly.

Table \ref{tab:exhaustive} in the Appendix illustrates our estimated highest performances achievable in $38$ of the most popular UCI and Kaggle classification and regression datasets. The number explanatory variables in these datasets varies from $d=3$ to $d=385$, the number of observations varies from $n=303$ to $n=583,250$ and the number of classes in classification experiment varies from $q=2$ to $q=26$.

\begin{figure}
        \includegraphics[width=0.45\textwidth]{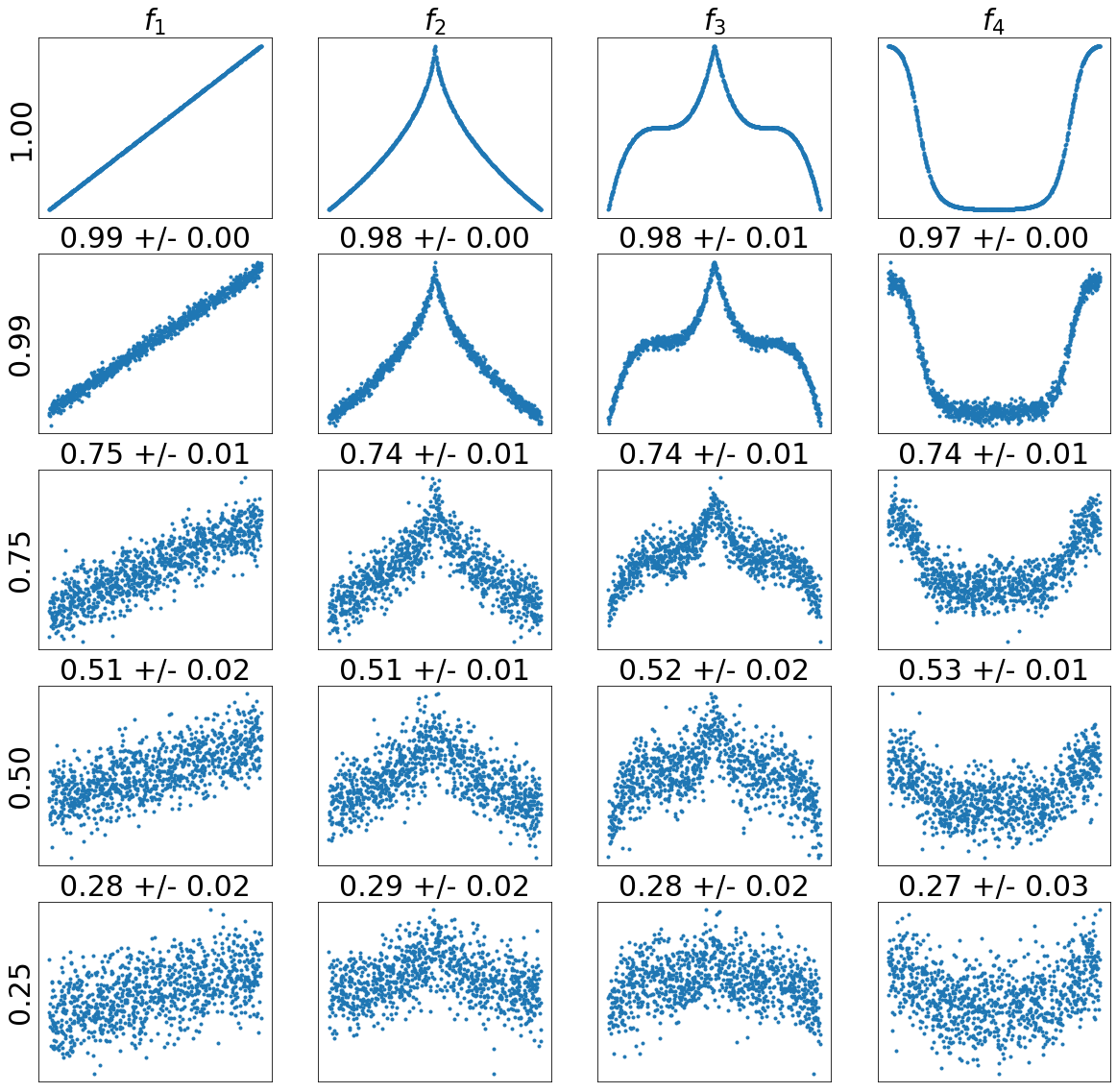}
        \caption{True theoretical-best (classic) $R^2$ ($y$-axis) and estimated theoretical-best (generalized) $R^2$ (upper $x$-axis) in the regression experiments of Section \ref{sct:app} for $d=1$,  for illustration purposes.}
      \label{fig:best_rmse}
\end{figure}
\begin{figure}
        \includegraphics[width=0.45\textwidth]{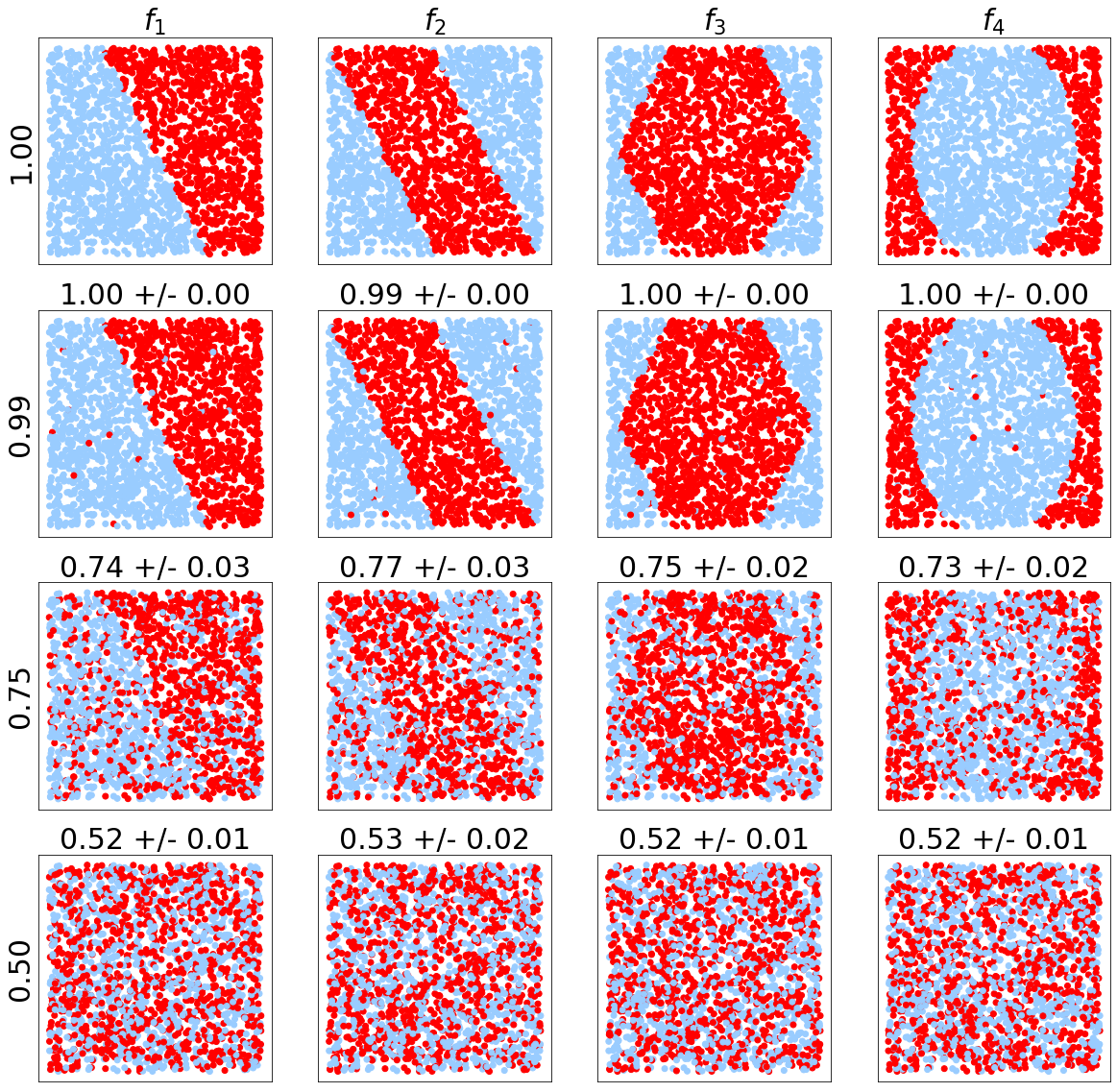} 
        \caption{True theoretical-best accuracy ($y$-axis) and estimated theoretical-best accuracy (upper $x$-axis) in the classification experiments of Section \ref{sct:app}, for $d=2$.}
       \label{fig:best_rsq}
\end{figure}

\begin{table*}[t]
\centering
\resizebox{0.7\textwidth}{!}{
\begin{tabular}{lcccc}
\hline
\hline
Ground Truth  & $f_1$ & $f_2$ & $f_3$ & $f_4$  \\
\hline
\hline
             Regression  &  &  &  &   \\
\hline
\hline
               &  & $R^2 (d=10)$ &  &   \\
\hline
$0.99$ &  $\bm{0.99 \pm 0.00}$ &  $\bm{0.99 \pm 0.00}$ &  $0.95 \pm 0.00$ &  $0.98 \pm 0.00$ \\ 
$0.75$ &  $\bm{0.73 \pm 0.01}$ &  $\bm{0.72 \pm 0.01}$ &  $0.64 \pm 0.02$ &  $\bm{0.73 \pm 0.01}$ \\ 
$0.50$ &  $\bm{0.49 \pm 0.01}$ &  $\bm{0.47 \pm 0.02}$ &  $0.41 \pm 0.01$ &  $\bm{0.48 \pm 0.01}$ \\ 
$0.25$ &  $\bm{0.25 \pm 0.01}$ &  $\bm{0.24 \pm 0.01}$ &  $\bm{0.21 \pm 0.02}$ &  $\bm{0.25 \pm 0.02}$ \\ 
\hline
               &  &  RMSE ($d=10$) &  &   \\
\hline
$0.10$ &  $\bm{0.11 \pm 0.00}$ &  $0.12 \pm 0.00$ &  $0.22 \pm 0.01$ &  $0.13 \pm 0.00$ \\ 
$0.58$ &  $\bm{0.60 \pm 0.01}$ &  $\bm{0.61 \pm 0.01}$ &  $0.69 \pm 0.02$ &  $0.60 \pm 0.00$ \\ 
$1.00$ &  $\bm{1.02 \pm 0.01}$ &  $\bm{1.03 \pm 0.02}$ &  $1.08 \pm 0.01$ &  $\bm{1.01 \pm 0.01}$ \\ 
$1.73$ &  $\bm{1.74 \pm 0.02}$ &  $\bm{1.75 \pm 0.02}$ &  $1.77 \pm 0.01$ &  $\bm{1.73 \pm 0.03}$ \\ 
\hline
\hline
             Classification  &  &  &  &   \\
\hline
\hline
               &  & Accuracy ($d=10$) &  &   \\
\hline
$1.00$ &  $\bm{0.99 \pm 0.00}$ &  $0.96 \pm 0.00$ &  $\bm{0.99 \pm 0.00}$ &  $\bm{0.99 \pm 0.00}$ \\ 
$0.99$ &  $\bm{0.97 \pm 0.03}$ &  $\bm{0.90 \pm 0.10}$ &  $\bm{0.98 \pm 0.00}$ &  $\bm{0.98 \pm 0.00}$ \\ 
$0.75$ &  $\bm{0.74 \pm 0.03}$ &  $\bm{0.67 \pm 0.05}$ &  $\bm{0.73 \pm 0.02}$ &  $\bm{0.74 \pm 0.02}$ \\ 
$0.50$ &  $0.57 \pm 0.03$ &  $\bm{0.55 \pm 0.03}$ &  $0.54 \pm 0.02$ &  $\bm{0.55 \pm 0.03}$ \\ 
\hline
\end{tabular}}
\caption{Comparison between true theoretical-best (classic) metrics and estimated theoretical-best (generalized) metrics, as described in Section \ref{sct:app} for $d=10$. Estimated metrics are represented as mean $\pm$ one standard-deviation. Bold entries correspond to cases where the true (classic) theoretical-best value is within two estimation standard deviations of the mean estimated (generalized) theoretical-best.}
\label{tab:summary}
\end{table*}

\textbf{Lean Model Building Experiments}: Good early-termination should result in low-regret, and low opportunity cost.  Regret is the percentage of models that were terminated that would have had a test performance higher than the estimated theoretical-best. The opportunity cost is the reduction in resource consumption that we would have incurred had we used early-termination.  If the estimated theoretical-best overshoots, the regret will be low but the opportunity cost will be high. If the estimated theoretical-best undershoots, the regret will be high, but the opportunity cost will be low.  A good data valuation estimation provides a good trade-off between regret and opportunity cost. 

To illustrate this tradeoff, we simulated applying early-termination in $100$ experiments on a the Don't Overfit ii Kaggle experiment using TensorFlow.  We did an $80$-$20$ split of the data $100$ times and use as model a $20 \times 20 \times 20 \times 20 \times 20 \times 1$ fully-connected neural classifier with ReLu inner layer activation, linear output layer activation, and binary cross-entropy loss. We train the model for $1000$ epochs,  and simulate earl-termination by implementing a TensorFlow callback. Termination is triggered when the running accuracy exceeds the estimated theoretical best ($82\%$). In the ex-post analysis, we consider that a model was overfitted when its held-out performance is at least $10\%$ worse than its training performance. Overall, $76\%$ of experiments were overfitted,  all of which would have been stopped by our early-termination rule, resulting in a $74\%$ reduction in runtime (a proxy for compute spent). Additionally, no experiment that did not overfit was stopped, and therefore the regret was null.

\textbf{Model-Free Variable Selection Case Study:} We illustrate the efficacy of our greedy model-free variable selection algorithm on the UCI Bank Note dataset. We first provide an intuitive qualitative analysis, then we verify that our model-free variable selection algorithm is consistent with our findings. The problem consists of determining whether a bank note is a forgery from properties of an image thereof, namely its \emph{entropy}, \emph{kurtosis}, \emph{skewness} and \emph{variance}. All $4$ variables are normalized to take value between $0$ and $1$ to ease illustration. \footnote{To be specific, we apply the transformation $x \to (x-x_{\text{min}})/(x_{\text{max}}-x_{\text{min}})$ to each variable.}

To determine which variable is the most insightful when used by itself to predict the label or, equivalently, the first variable our algorithm should be selecting, we generate a scatter plot of values of each variable color-coded with the type of note, green for authentic notes and red for forgeries. This is illustrated in Figure \ref{fig:var1} where it can be seen that it is visually very hard to differentiate genuine bank notes from forgeries solely using the \emph{entropy} variable. As for the \emph{kurtosis} variable, while a normalized \emph{kurtosis} higher than $0.6$ is a strong indication that the bank note is a forgery, this only happens about $7\%$ of the time. When the \emph{kurtosis} is lower than $0.6$ on the other hand, it is very hard to distinguish genuine notes from forgeries using the \emph{kurtosis} variable alone.  The \emph{skewness} variable is visually more useful than both \emph{kurtosis} and \emph{entropy}, but the \emph{variance} variable is clearly the most insightful explanatory variable. Genuine bank notes tend to have a higher \emph{variance} than forgeries.
\begin{figure}
        \includegraphics[width=0.45\textwidth]{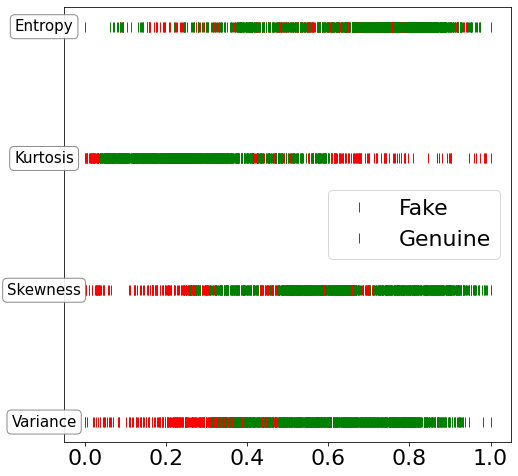} 
        \caption{Scatter plot of explanatory variables in the UCI Bank Note dataset. Values are rescaled to take values in $[0, 1]$ to ease illustration.}
       \label{fig:var1}
\end{figure}

To figure out which of the three remaining explanatory variables would complement the \emph{variance} variable the most, we make three 2D scatter plots with \emph{variance} as the x-axis and the other input as the y-axis and, as we did before, we color dots green (resp. red) when the associated inputs came from a genuine (resp. fake) bank note. Intuitively, the explanatory variable that complements the variance variable the most is the one whose green and red clusters of points are the most distinguishable. The more distinguishable these two clusters are, the more accurately we can predict whether the bank note is a forgery. The more the two collections overlap, the more ambiguous our prediction will be. As it can be seen in Figure \ref{fig:var2}, the explanatory variable that, when used in conjunction with the \emph{variance} variable, separates genuine and fake notes the most is \emph{skewness}.
\begin{figure}
        \includegraphics[width=0.48\textwidth]{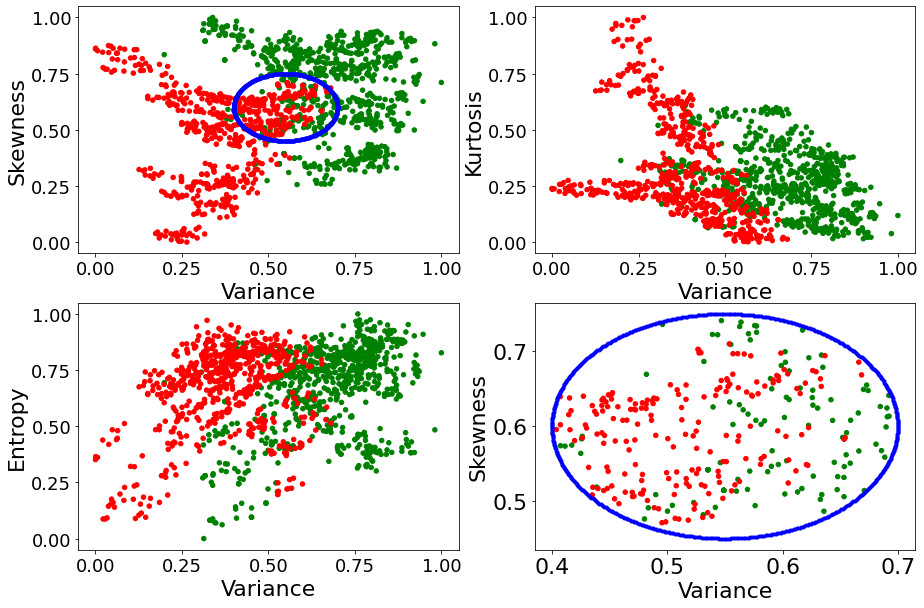} 
        \caption{2D scatter plot of the variance (x-axis) against each other explanatory variable (y-axis) in the UCI Bank Note dataset.  Values are rescaled to take values in $[0, 1]$ to ease illustration. The bottom-right plot is a zoomed-in version of the top-left plot around the blue ellipse.}
       \label{fig:var2}
\end{figure}

To qualitatively determine which of \emph{entropy} or \emph{kurtosis} would complement the pair (\emph{variance,} \emph{skewness}) the most, we identity values of the pair (\emph{variance}, \emph{skewness}) that are jointly inconclusive about whether the bank note is a forgery. This is the region of the \emph{variance} x \emph{skewness} plane where green dots and red dots overlap. We have crudely identified this region in the top-left plot in Figure \ref{fig:var2} with a blue ellipse, a zoomed-in version thereof is displayed in the bottom-right plot. We then attempt to determine which of \emph{entropy} and \emph{kurtosis} can best help alleviate the ambiguity inherent to that region. To do so, we consider all the bank notes that fall within the blue ellipse above, and we plot them on the four planes \emph{variance} x \emph{kurtosis}, \emph{variance} x \emph{entropy}, \emph{skewness} x \emph{kurtosis}, and \emph{skewness} x \emph{entropy}, in an attempt to figure out at a glance how much ambiguity we can remove by knowing the value of the \emph{entropy} or \emph{kurtosis} variable. This is illustrated in Figure \ref{fig:var3}, where it can be seen that the addition of the \emph{kurtosis} explanatory variable is sufficient to classify all bank notes almost perfectly, while the \emph{entropy} variable is not sufficient to remove all ambiguity.
\begin{figure}
        \includegraphics[width=0.48\textwidth]{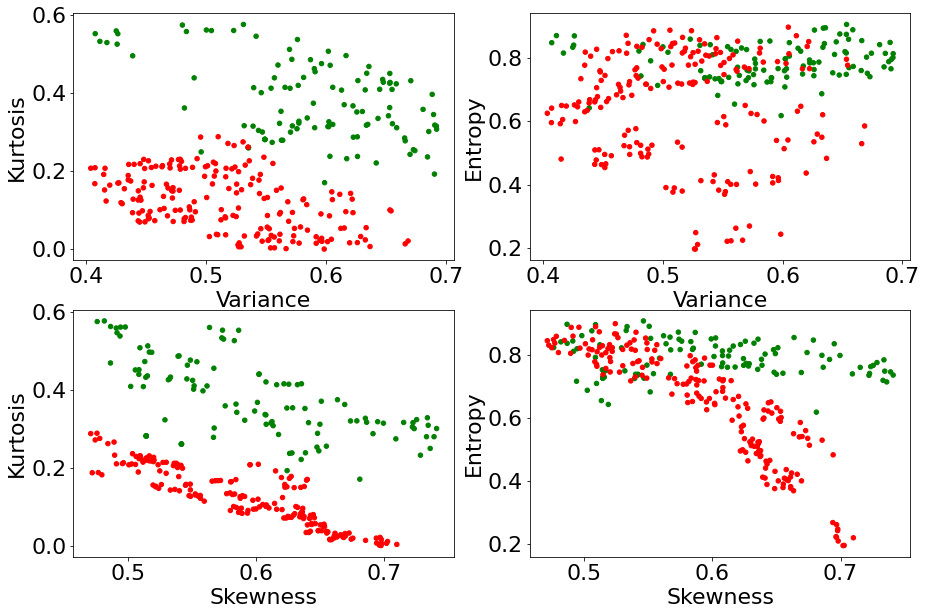} 
        \caption{2D scatter plots of bank notes that fall in the ambiguity ellipse of Figure \ref{fig:var2} --- i.e. that can hardly be classified as genuine or fake using Variance and Skewness alone. The x-axis is either Variance or Skewness and the y-axis is either Entropy or Kurtosis.}
       \label{fig:var3}
\end{figure}

To recap,  our greedy model-free variable selection algorithm applied to the UCI Bank Note dataset should first select \emph{variance} as the most insightful explanatory variable, then \emph{skewness} as the explanatory variable that complements \emph{variance} the most, and finally \emph{kurtosis}. \emph{Entropy} doesn't add much value to the other $3$,  and using the other \emph{variance}, \emph{skewness} and \emph{kurtosis}, we can achieve perfect accuracy. This is indeed what our greedy model-free variable selection does, as illustrated in Table \ref{tab:vs_bn}.

\begin{table*}[!ht]
\centering
\resizebox{0.7\textwidth}{!}{
\begin{tabular}{rlll}
\toprule
 Selection Order &      Variable & Running Achievable $R^2$ & Running Achievable Accuracy \\
\midrule
               1 &   Variance &                           0.51 &                  0.90 \\
               2 &  Skewness &                         0.58 &                  0.93 \\
               3 &  Kurtosis &                            0.75 &                  1.00\\
               4 &  Entropy &                             0.75&                  1.00 \\
\bottomrule
\end{tabular}}
\caption{Greedy model-free variable selection based on theoretical-best performance achievable, and applied to the UCI Bank Note dataset.}
\label{tab:vs_bn}
\end{table*}

We further illustrate our greedy model-free variable selection algorithm on a regression problem with a much larger set of explanatory variables, namely the Kaggle house price advanced regression dataset. This dataset has $80$ explanatory variables, almost evenly split between categorical and continuous variables.  Results are presented in Table \ref{tab:vs_house} in the Appendix, where it can be seen that the relative importance of the top-20 and bottom-20 variables selected makes intuitive sense.

\textbf{Lean Model Improvement Experiments:} Attempts to improve a production model can be grouped into two categories: model-driven attempts and data-driven attempts. Model-driven attempts aim at improving the production model by looking for a model using the \emph{same} explanatory variables, but that has a better fit (i.e. that approximates the true conditional distribution $P_{\bm{y} \vert \bm{x}}$ better than the production model does). Data-driven attempts aim at boosting the performance of the production model by looking for new and complementary explanatory variables from which new insights could be generated.

Consistent with the LeanML design pattern, prior to any data-driven attempt at improving a production model, it is crucial to quantify the highest performance boost that the new set of explanatory variables can bring about. It might be counter-intuitive, but explanatory variables that may boost the performance of a production model are not necessarily directly informative about the business outcome itself; in fact they can be independent from the business outcome.  Good candidates should be informative about the business outcome \emph{conditional on existing explanatory variables}. To illustrate this point, let us consider the regression generative model $y = ix_1 + (1-i)x_0$, where $x_1$ and $x_0$ are i.i.d. and $i$ is a Bernoulli variable independent from both $x_1$ and $x_0$.  It is easy to see that $i$ and $y$ are independent, as $P_{y \vert i = 1} = P_{y \vert i = 0} = P_{x_1} = P_{x_0}=P_y$. As such, $i$ contains no insight about $y$.  Additionally, knowing $x_1$ and $x_0$ helps predict $y$, but $y$ cannot be predicted perfectly using $x_1$ and $x_0$ alone. However, once we know $x_1$ and $x_0$, using $i$ as explanatory variable allows us to predict $y$ perfectly.  Thus, being informative about a business outcome should not be a requirement for explanatory variables to use to improve a production model in a data-driven fashion. 

Similarly, because a new explanatory variable is highly informative about the business outcome of interest, does not mean it should be used to improve a production model: the new explanatory variable may very well be redundant with respect to the explanatory variables used to train the production model. To illustrate this, we estimate the highest performance achievable in the previous UCI Bank Note experiment without the \emph{variance} explanatory variable, which we recall we previously found to be the variable that was the most insightful about the business outcome to predict (when used by itself). We find that the outcome can be predicted with a $99\%$ accuracy, even without the \emph{variance} explanatory variable (i.e. using \emph{skewness}, \emph{kurtosis} and \emph{entropy}). Table \ref{tab:vs_bn_novar} contains the result of our model-free variable selection algorithm applied to all explanatory variables but \emph{variance}.

No matter the number of explanatory variables or features a production model was trained with, no matter the number of newly available explanatory variables or features,  by substracting the theoretical-best performances achievable using the old set of explanatory variables or features from the theoretical-best performances achievable using the old and new sets combined, we get the highest performance boost the new set of explanatory variables or features may bring about.

\begin{table*}[!ht]
\centering
\resizebox{0.7\textwidth}{!}{
\begin{tabular}{rlll}
\toprule
 Selection Order &      Variable & Running Achievable $R^2$ & Running Achievable Accuracy \\
\midrule
               1 &  Skewness &                         0.38 &                  0.83 \\
               2 &  Entropy &                            0.45 &                  0.87 \\
               3 &  Kurtosis &                           0.74 &                  0.99 \\
\bottomrule
\end{tabular}}
\caption{Greedy model-free variable selection based on theoretical-best performance achievable, and applied to the UCI Bank Note dataset excluding the variance explanatory variable..}
\label{tab:vs_bn_novar}
\end{table*}

As for model-driven attempts at improving a production model, to determine the feasibility of such endeavors,  we may simply compare the performance of the production model out-of-sample to the estimated theoretical-best. A production model can be improved in a purely model-driven fashion if and only if its performance is smaller than the theoretical-best, and the gap between the two, which we refer to as the sub-optimality gap, is the performance boost we stand to gain by simply looking for better models. Prior to such model-driven attempts, data scientists should first quantify the sub-optimality gap, and question whether the potential business impact outweighs the resources needed to look for better models.

\section{Conclusion}
We provide a design pattern for machine learning projects which we refer to as the LeanML design pattern. The LeanML design pattern is a framework for structuring predictive modeling projects that empowers data scientists to slash avoidable wastes of time and compute resources. The LeanML design pattern implements two very intuitive key principles, which we refer to as the LeanML principles, namely that: one should always condition the running of a machine learning experiment on estimating its feasibility, and one should always pro-actively terminate an experiment one started as soon as one can reliably determine it will fail.  What enables LeanML is the realization that it is possible to estimate the best performance one may achieve when predicting an outcome $\displaystyle \vy \in \mathcal{Y}$ using a given set of explanatory variables $\displaystyle \vx \in \mathcal{X}$ for a wide range of metrics, without training any predictive model, and that doing so is in fact easier, faster, and cheaper than learning the best predictive model. We provide theoretical results expressing the theoretical-best $R^2$, MSE, classification accuracy and log-likelihood per observation, as a function of the mutual information $I\left(\bm{y}; \bm{x}\right)$ and (occasionally) a measure of the variability of $\bm{y}$. We illustrate the efficacy of LeanML on a wide range of synthetic and real-life experiments.

\textbf{Code:} The LeanML design pattern may be seamlessly implemented using the Function-As-A-Service product KXY. KXY is accessible through the \textbf{kxy} Python package on PyPi (\textbf{pip install kxy}) or GitHub (\url{https://github.com/kxytechnologies/kxy-python}), or through the KXY REST API.  The product is free for academic use.  

Experiments in this paper may be reproduced using the GitHub repo \url{https://github.com/kxytechnologies/kxy-datasets}.

\newpage

\bibliography{lean_ml}
\bibliographystyle{lean_ml}
\newpage

\appendix

\begin{table*}[t]
\centering
\resizebox{\textwidth}{!}{
\begin{tabular}{ c |c |c }
                               & $y$ \bf{ is continuous} & $y$ \bf{is categorical} \\ 
\hline 
$\displaystyle \vx$ \bf{is continuous} & $I(y; \displaystyle \vx) = h(y) + h(\displaystyle \vx) - h(y; \displaystyle \vx)$ & $I(y; \displaystyle \vx) = h(\displaystyle \vx) - \sum_{i\in \mathcal{Y}} h(\displaystyle \vx | y=i)P_y(i)$ \\  
\hline 
$\displaystyle \vx$ \bf{is categorical} & $I(y; \displaystyle \vx) = h(y) - \sum_{i\in \mathcal{X}} h(y | \displaystyle \vx=i)P_{\displaystyle \vx}(i)$ & $I(y; \displaystyle \vx) = H(y) + H(\displaystyle \vx) - H(y; \displaystyle \vx)$ \\
\hline 
\vtop{\hbox{\strut $\displaystyle \vx$ \bf{has continous  coordinates} $\displaystyle \vx_c$}\hbox{\strut \bf{and categorical coordinates} $\displaystyle \vx_d$}}& $I(y; \displaystyle \vx) = h(y) + \sum_{i \in \mathcal{X}_d} \left[h(\displaystyle \vx_c | \displaystyle \vx_d = i) -  h\left(y, \displaystyle \vx_c | \displaystyle \vx_d = i\right)\right]P_{\displaystyle \vx_d}(i)$ & $I(y; \displaystyle \vx) = I(y; \displaystyle \vx_d) + \sum_{i \in \mathcal{X}_d} P_{\displaystyle \vx_d}(i) h(\displaystyle \vx_c | \displaystyle \vx_d = i) -  \sum_{j \in \mathcal{Y}} h\left(\displaystyle \vx_c | \displaystyle \vx_d = i, y=j\right)P_{\displaystyle \vx_d, y}(i, j)$ 
\end{tabular}}
\caption{Expression of the mutual information $I(y; \displaystyle \vx)$ as a function of the Shannon entropy $H(.)$, and/or the differential entropy $h(.)$, depending on whether $y$ and/or $\displaystyle \vx$ has continuous and/or categorical coordinates. Expressions of the type $h(\displaystyle \vx|y=i)$ are to be understood as the differential entropy of the continuous conditional distribution $\displaystyle \vx|y=i$.}
\label{tab:mutual_information}
\end{table*}

\begin{table*}[!ht]
\centering
\resizebox{0.7\textwidth}{!}{
\begin{tabular}{lcccc}
\hline
\hline
Exact $R^2$   & $f_1$ & $f_2$ & $f_3$ & $f_4$  \\
\hline
$d=1$               &  &  &  &   \\
\hline
$0.99$ &  $\bm{0.99 \pm 0.00}$ &  $\bm{0.98 \pm 0.00}$ &  $\bm{0.98 \pm 0.01}$ &  $0.97 \pm 0.00$ \\ 
$0.75$ &  $\bm{0.75 \pm 0.01}$ &  $\bm{0.74 \pm 0.01}$ &  $\bm{0.74 \pm 0.01}$ &  $\bm{0.74 \pm 0.01}$ \\ 
$0.50$ &  $\bm{0.51 \pm 0.02}$ &  $\bm{0.51 \pm 0.01}$ &  $\bm{0.52 \pm 0.02}$ &  $\bm{0.53 \pm 0.01}$ \\ 
$0.25$ &  $\bm{0.28 \pm 0.02}$ &  $\bm{0.29 \pm 0.02}$ &  $\bm{0.28 \pm 0.02}$ &  $\bm{0.27 \pm 0.03}$ \\ 
\hline
$d=2$               &  &  &  &   \\
\hline
$0.99$ &  $\bm{0.99 \pm 0.00}$ &  $\bm{0.99 \pm 0.00}$ &  $0.97 \pm 0.00$ &  $0.98 \pm 0.00$ \\ 
$0.75$ &  $\bm{0.75 \pm 0.01}$ &  $\bm{0.75 \pm 0.01}$ &  $\bm{0.71 \pm 0.02}$ &  $\bm{0.74 \pm 0.01}$ \\ 
$0.50$ &  $\bm{0.52 \pm 0.03}$ &  $\bm{0.52 \pm 0.02}$ &  $\bm{0.48 \pm 0.02}$ &  $\bm{0.51 \pm 0.02}$ \\ 
$0.25$ &  $\bm{0.29 \pm 0.02}$ &  $\bm{0.29 \pm 0.02}$ &  $\bm{0.27 \pm 0.02}$ &  $\bm{0.31 \pm 0.04}$ \\
\hline
$d=5$               &  &  &  &   \\
\hline
$0.99$ &  $\bm{0.99 \pm 0.00}$ &  $\bm{0.99 \pm 0.00}$ &  $0.96 \pm 0.01$ &  $0.98 \pm 0.00$ \\ 
$0.75$ &  $\bm{0.73 \pm 0.01}$ &  $\bm{0.73 \pm 0.01}$ &  $0.64 \pm 0.04$ &  $\bm{0.72 \pm 0.01}$ \\ 
$0.50$ &  $\bm{0.47 \pm 0.01}$ &  $\bm{0.47 \pm 0.01}$ &  $\bm{0.44 \pm 0.03}$ &  $\bm{0.48 \pm 0.02}$ \\ 
$0.25$ &  $\bm{0.25 \pm 0.01}$ &  $\bm{0.23 \pm 0.01}$ &  $0.21 \pm 0.01$ &  $\bm{0.24 \pm 0.01}$ \\ 
\hline
$d=10$               &  &  &  &   \\
\hline
$0.99$ &  $\bm{0.99 \pm 0.00}$ &  $\bm{0.99 \pm 0.00}$ &  $0.95 \pm 0.00$ &  $0.98 \pm 0.00$ \\ 
$0.75$ &  $\bm{0.73 \pm 0.01}$ &  $\bm{0.72 \pm 0.01}$ &  $0.64 \pm 0.02$ &  $\bm{0.73 \pm 0.01}$ \\ 
$0.50$ &  $\bm{0.49 \pm 0.01}$ &  $\bm{0.47 \pm 0.02}$ &  $0.41 \pm 0.01$ &  $\bm{0.48 \pm 0.01}$ \\ 
$0.25$ &  $\bm{0.25 \pm 0.01}$ &  $\bm{0.24 \pm 0.01}$ &  $\bm{0.21 \pm 0.02}$ &  $\bm{0.25 \pm 0.02}$ \\ 
\hline
\end{tabular}}
\caption{Comparison between true theoretical-best (classic) regression $R^2$ and estimated theoretical-best (generalized) regression $R^2$, as described in Section \ref{sct:app} for various values of $d$. Estimated $R^2$ are represented as mean $\pm$ one standard-deviation. Bold entries correspond to cases where the true (classic) theoretical-best value is within two estimation standard deviations of the mean estimated (generalized) theoretical-best.}
\label{tab:full_rsq}
\end{table*}

\begin{table*}[!ht]
\centering
\resizebox{0.7\textwidth}{!}{
\begin{tabular}{lcccc}
\hline
\hline
Exact RMSE  & $f_1$ & $f_2$ & $f_3$ & $f_4$  \\
\hline
$d=1$               &  &  &  &   \\
\hline
$0.10$ &  $0.12 \pm 0.00$ &  $0.13 \pm 0.01$ &  $\bm{0.14 \pm 0.03}$ &  $0.18 \pm 0.01$ \\ 
$0.58$ &  $\bm{0.58 \pm 0.01}$ &  $\bm{0.58 \pm 0.01}$ &  $\bm{0.59 \pm 0.02}$ &  $\bm{0.59 \pm 0.01}$ \\ 
$1.00$ &  $\bm{0.98 \pm 0.02}$ &  $\bm{0.99 \pm 0.02}$ &  $\bm{0.98 \pm 0.02}$ &  $\bm{0.98 \pm 0.02}$ \\ 
$1.73$ &  $\bm{1.68 \pm 0.02}$ &  $\bm{1.69 \pm 0.04}$ &  $\bm{1.71 \pm 0.03}$ &  $\bm{1.69 \pm 0.05}$ \\ 
\hline
$d=2$               &  &  &  &   \\
\hline
$0.10$ &  $\bm{0.11 \pm 0.00}$ &  $\bm{0.11 \pm 0.00}$ &  $0.18 \pm 0.01$ &  $0.15 \pm 0.00$ \\ 
$0.58$ &  $\bm{0.58 \pm 0.01}$ &  $\bm{0.58 \pm 0.01}$ &  $\bm{0.62 \pm 0.02}$ &  $\bm{0.59 \pm 0.01}$ \\ 
$1.00$ &  $\bm{0.98 \pm 0.03}$ &  $\bm{0.98 \pm 0.02}$ &  $\bm{1.03 \pm 0.01}$ &  $\bm{0.98 \pm 0.02}$ \\ 
$1.73$ &  $\bm{1.69 \pm 0.03}$ &  $\bm{1.68 \pm 0.04}$ &  $\bm{1.70 \pm 0.02}$ &  $\bm{1.67 \pm 0.04}$ \\ 
\hline
$d=5$               &  &  &  &   \\
\hline
$0.10$ &  $\bm{0.11 \pm 0.00}$ &  $0.12 \pm 0.00$ &  $0.20 \pm 0.01$ &  $0.13 \pm 0.00$ \\ 
$0.58$ &  $\bm{0.60 \pm 0.01}$ &  $\bm{0.60 \pm 0.01}$ &  $0.69 \pm 0.04$ &  $\bm{0.61 \pm 0.01}$ \\ 
$1.00$ &  $\bm{1.02 \pm 0.01}$ &  $\bm{1.03 \pm 0.01}$ &  $\bm{1.05 \pm 0.03}$ &  $\bm{1.02 \pm 0.02}$ \\ 
$1.73$ &  $\bm{1.74 \pm 0.02}$ &  $1.76 \pm 0.01$ &  $\bm{1.78 \pm 0.03}$ &  $\bm{1.74 \pm 0.01}$ \\ 
\hline
$d=10$               &  &  &  &   \\
\hline
$0.10$ &  $\bm{0.11 \pm 0.00}$ &  $0.12 \pm 0.00$ &  $0.22 \pm 0.01$ &  $0.13 \pm 0.00$ \\ 
$0.58$ &  $\bm{0.60 \pm 0.01}$ &  $\bm{0.61 \pm 0.01}$ &  $0.69 \pm 0.02$ &  $0.60 \pm 0.00$ \\ 
$1.00$ &  $\bm{1.02 \pm 0.01}$ &  $\bm{1.03 \pm 0.02}$ &  $1.08 \pm 0.01$ &  $\bm{1.01 \pm 0.01}$ \\ 
$1.73$ &  $\bm{1.74 \pm 0.02}$ &  $\bm{1.75 \pm 0.02}$ &  $1.77 \pm 0.01$ &  $\bm{1.73 \pm 0.03}$ \\ 
\hline
\end{tabular}}
\caption{Comparison between true theoretical-best (classic) RMSE and estimated theoretical-best (generalized) RMSE, as described in Section \ref{sct:app} for various values of $d$. Estimated RMSE are represented as mean $\pm$ one standard-deviation. Bold entries correspond to cases where the true (classic) theoretical-best value is within two estimation standard deviations of the mean estimated (generalized) theoretical-best. }
\label{tab:full_rmse}
\end{table*}

\begin{table*}[!ht]
\centering
\resizebox{0.7\textwidth}{!}{
\begin{tabular}{lcccc}
\hline
\hline
Exact  & $f_1$ & $f_2$ & $f_3$ & $f_4$  \\
\hline
$d=1$               &  &  &  &   \\
\hline
$1.00$ &  $0.98 \pm 0.00$ &  $\bm{1.00 \pm 0.00}$ &  $\bm{1.00 \pm 0.00}$ &  $0.98 \pm 0.00$ \\ 
$0.99$ &  $\bm{0.98 \pm 0.02}$ &  $\bm{0.99 \pm 0.01}$ &  $\bm{0.98 \pm 0.02}$ &  $\bm{0.97 \pm 0.02}$ \\ 
$0.75$ &  $\bm{0.74 \pm 0.03}$ &  $\bm{0.77 \pm 0.02}$ &  $\bm{0.76 \pm 0.03}$ &  $\bm{0.74 \pm 0.02}$ \\ 
$0.50$ &  $\bm{0.52 \pm 0.01}$ &  $\bm{0.51 \pm 0.01}$ &  $\bm{0.52 \pm 0.01}$ &  $0.52 \pm 0.01$ \\ 
\hline
$d=2$               &  &  &  &   \\
\hline
$1.00$ &  $\bm{1.00 \pm 0.00}$ &  $\bm{1.00 \pm 0.00}$ &  $\bm{1.00 \pm 0.00}$  &  $\bm{1.00 \pm 0.00}$ \\ 
$0.99$ &  $1.00 \pm 0.00$          &  $\bm{0.99 \pm 0.00}$ &  $\bm{1.00 \pm 0.00}$ &  $\bm{1.00 \pm 0.00}$ \\ 
$0.75$ &  $\bm{0.74 \pm 0.03}$ &  $\bm{0.77 \pm 0.03}$ &  $\bm{0.75 \pm 0.02}$ &  $\bm{0.73 \pm 0.02}$ \\ 
$0.50$ &  $\bm{0.52 \pm 0.01}$ &  $\bm{0.53 \pm 0.02}$ &  $0.52 \pm 0.01$          &  $0.52 \pm 0.01$ \\ 
\hline
$d=5$               &  &  &  &   \\
\hline
$1.00$ &  $\bm{1.00 \pm 0.00}$ &  $\bm{1.00 \pm 0.00}$ &  $\bm{1.00 \pm 0.00}$ &  $\bm{1.00 \pm 0.00}$ \\ 
$0.99$ &  $\bm{0.99 \pm 0.01}$ &  $\bm{0.99 \pm 0.01}$ &  $\bm{1.00 \pm 0.00}$ &  $\bm{0.99 \pm 0.00}$ \\ 
$0.75$ &  $\bm{0.76 \pm 0.02}$ &  $\bm{0.72 \pm 0.03}$ &  $\bm{0.76 \pm 0.03}$ &  $\bm{0.75 \pm 0.03}$ \\ 
$0.50$ &  $0.55 \pm 0.02$ &  $\bm{0.54 \pm 0.03}$ &  $0.56 \pm 0.03$ &  $\bm{0.54 \pm 0.02}$ \\ 
\hline
$d=10$               &  &  &  &   \\
\hline
$1.00$ &  $\bm{0.99 \pm 0.00}$ &  $0.96 \pm 0.00$ &  $\bm{0.99 \pm 0.00}$ &  $\bm{0.99 \pm 0.00}$ \\ 
$0.99$ &  $\bm{0.97 \pm 0.03}$ &  $\bm{0.90 \pm 0.10}$ &  $\bm{0.98 \pm 0.00}$ &  $\bm{0.98 \pm 0.00}$ \\ 
$0.75$ &  $\bm{0.74 \pm 0.03}$ &  $\bm{0.67 \pm 0.05}$ &  $\bm{0.73 \pm 0.02}$ &  $\bm{0.74 \pm 0.02}$ \\ 
$0.50$ &  $0.57 \pm 0.03$ &  $\bm{0.55 \pm 0.03}$ &  $0.54 \pm 0.02$ &  $\bm{0.55 \pm 0.03}$ \\ 
\hline
\end{tabular}}
\caption{Comparison between true theoretical-best classification accuracy and estimated theoretical-best accuracy, as described in Section \ref{sct:app} for various values of $d$, when the uniform-informativeness condition is met. Estimated accuracies are represented as mean $\pm$ one standard-deviation. Bold entries correspond to cases where the true theoretical-best value is within two estimation standard deviations of the mean estimated theoretical-best.}
\label{tab:full_acc}
\end{table*}

\begin{table*}[!ht]
\centering
\resizebox{0.7\textwidth}{!}{
\begin{tabular}{rlll}
\toprule
 Selection Order &      Variable & Running Achievable $R^2$ & Running Achievable RMSE \\
\midrule

0               &    No Variable &                         0.00 &                7.94e+04 \\
1               &    OverallQual &                         0.65 &                4.70e+04 \\
2               &      GrLivArea &                         0.78 &                3.70e+04 \\
3               &      YearBuilt &                         0.84 &                3.17e+04 \\
4               &    TotalBsmtSF &                         0.85 &                3.12e+04 \\
5               &    OverallCond &                         0.85 &                3.08e+04 \\
6               &       MSZoning &                         0.85 &                3.08e+04 \\
7               &      BsmtUnfSF &                         0.85 &                3.08e+04 \\
8               &        LotArea &                         0.85 &                3.08e+04 \\
9               &     GarageCars &                         0.85 &                3.08e+04 \\
10              &     Fireplaces &                         0.85 &                3.03e+04 \\
11              &   GarageFinish &                         0.85 &                3.03e+04 \\
12              &   KitchenAbvGr &                         0.85 &                3.03e+04 \\
13              &  SaleCondition &                         0.85 &                3.03e+04 \\
14              &   Neighborhood &                         0.86 &                3.01e+04 \\
15              &         MoSold &                         0.86 &                3.00e+04 \\
16              &       2ndFlrSF &                         0.86 &                2.98e+04 \\
17              &      LandSlope &                         0.90 &                2.46e+04 \\
18              &     Foundation &                         0.93 &                2.03e+04 \\
19              &     BsmtFinSF1 &                         0.96 &                1.67e+04 \\
20              &          Alley &                         0.96 &                1.67e+04 \\

          \dots &        \dots &                         \dots &          \dots \\
          
60              &  BsmtFinType1 &                         1.00 &                1.75e+03 \\
61              &   MiscFeature &                         1.00 &                1.75e+03 \\
62              &    CentralAir &                         1.00 &                1.75e+03 \\
63              &      BldgType &                         1.00 &                1.75e+03 \\
64              &    GarageCond &                         1.00 &                1.75e+03 \\
65              &        YrSold &                         1.00 &                1.75e+03 \\
66              &        PoolQC &                         1.00 &                1.75e+03 \\
67              &      PoolArea &                         1.00 &                1.75e+03 \\
68              &     ExterQual &                         1.00 &                1.75e+03 \\
69              &      BsmtCond &                         1.00 &                1.75e+03 \\
70              &    MasVnrType &                         1.00 &                1.75e+03 \\
71              &      LotShape &                         1.00 &                1.75e+03 \\
72              &       Heating &                         1.00 &                1.75e+03 \\
73              &    MasVnrArea &                         1.00 &                1.75e+03 \\
74              &  BsmtExposure &                         1.00 &                1.75e+03 \\
75              &  BsmtFullBath &                         1.00 &                1.75e+03 \\
76              &        Street &                         1.00 &                1.75e+03 \\
77              &         Fence &                         1.00 &                1.75e+03 \\
78              &  TotRmsAbvGrd &                         1.00 &                1.75e+03 \\
79              &     3SsnPorch &                         1.00 &                1.75e+03 \\

\bottomrule
\end{tabular}}
\caption{Greedy model-free variable selection based on theoretical-best performance achievable, and applied to the Kaggle house prices advanced regression techniques dataset. Illustrated are the top-$20$ and bottom-$20$ variables selected.}
\label{tab:vs_house}
\end{table*}

\section{Mutual Information Estimation: Relation Between the MIND, MINE and NWJ Estimators}
\label{sct:mi_discussion}
The fundamental limitation of MINE (\cite{belghazi2018mutual}) and NWJ (\cite{nguyen2010estimating}) as mutual information estimators is that, by assuming that we can reliably estimate expectations of the form $E\left[T(\bm{y}, \bm{x}) \right]$ from our data for any function $T$, they implicitly assume that we have enough data to fully characterize the joint distribution $P_{\bm{y}, \bm{x}}$. The same can be said of the CPC model of \cite{oord2018representation}.

This is problematic because the mutual information itself is only a loose property of the joint distribution. For instance, the mutual information does not depend on marginal distributions, it is invariant by 1-to-1 transformations, and the same mutual information value can be accounted for by a large number of copula distributions. In order for us to reliably estimate all expectations of the form $E\left[T(\bm{y}, \bm{x}) \right]$ and $E\left[e^{T(\bm{y}, \bm{x})} \right]$, as required by NWJ and MINE,  we need an excessively large sample size to achieve a reasonably small variance.

Fortunately, both follow MINE and NWJ are based on variational characterizations of the KL divergence between two distributions. MINE uses the Donsker-Varadhan bound (\cite{donsker1975asymptotic}) $$KL\left(P \vert \vert Q\right) =  \sup_{T \in L^\infty (Q)} E_P(T) - \log E_Q\left( e^T\right)$$ and \cite{nguyen2010estimating} proposed their own bound $$KL\left(P \vert \vert Q\right) = \sup_{T \in L^\infty (Q)}  E_P\left( T\right) - E_Q\left( e^{T-1}\right).$$ Rather than follow \cite{belghazi2018mutual} and \cite{nguyen2010estimating} and directly estimate the mutual information in the primal space as $$I\left(\bm{y}; \bm{x} \right) = KL\left( P_{\bm{y}, \bm{x}} \vert \vert  P_{\bm{y}} \otimes P_{\bm{x}} \right),$$ we may estimate the mutual information in the copula-uniform dual space\footnote{i.e. the image of the primal/input space by the probability integral transform.} by noting that $$I\left(\bm{y}; \bm{x}\right) = h(\bm{u}_{\bm{y}}) + h(\bm{u}_{\bm{x}}) - h(\bm{u}_{\bm{y}}, \bm{u}_{\bm{x}}), $$and that a copula entropy $h(\bm{u}_{\bm{z}})$ is nothing but the opposite of the KL-divergence between the copula distribution of $\bm{z}$ and the standard uniform distribution: $$h\left(\bm{u}_{\bm{z}} \right) = -KL\left( P_{\bm{u}_{\bm{z}}}  \vert \vert U \right).$$ We may then use the variational characterizations above to estimate copula entropies. 

In practice, $T$ is taken in a parametric space of functions,  $T_\theta \in \mathcal{T}_{\Theta}$, and the copula-entropy estimators read
$$ h_{DV}\left( \bm{u}_{\bm{z}} \right) = \inf_{\theta \in \Theta} -E\left( T_{\theta} \left( \bm{u}_{\bm{z}} \right) \right) + \log \int_{[0, 1]^d} e^{T_\theta(\bm{u})} d\bm{u}$$
and
$$h_{NWJ}\left( \bm{u}_{\bm{z}} \right) = \inf_{\theta \in \Theta} -E\left( T_{\theta}\left( \bm{u}_{\bm{z}} \right) \right) + \int_{[0, 1]^d} e^{T_\theta(\bm{u})-1} d\bm{u}.$$
A direct consequence of the results in \cite{pmlr-v130-kom-samo21a} is that, if $\mathcal{T}_{\Theta}$ is a finite dimensional RKHS with feature map $\phi$ containing an intercept term (i.e. $T_\theta \left( \bm{u} \right) = \theta_0 + \theta^T \phi\left( \bm{u} \right)$),  then $h_{DV}$ and $h_{NWJ}$ are the same, and are the unique solution to the MIND maximum-entropy problem
\begin{align}
\begin{cases}
\underset{P \in \mathcal{D}_d}{\max} ~~~ h\left( P \right)  \\
\text{s.t.} ~ E_P\left[ \phi \left(\displaystyle \vu \right)  \right] = E_{P_{\bm{u}_{\bm{z}}}}\left[ \phi \left(\displaystyle \vu \right)  \right]  \nonumber
\end{cases}.
\end{align}

This is the case for instance when $\mathcal{T}_{\Theta}$ is a neural network whose final layer is linear with an intercept term, and all other layer parameters are frozen. Note that, in this finite-dimensional RKHS case, the copula entropy estimator depends on the data distribution solely through the expectation $E_{P_{\bm{u}_{\bm{z}}}}\left[ \phi \left(\displaystyle \vu \right)  \right]$, which only needs to be evaluated once. Typically, $\phi$ would be chosen so that we may reliably estimate this expectation from the amount of data available.

Back to our neural network example, when none of the layers are frozen, both $h_{DV}$ and $h_{NWJ}$ are solutions to the minimax entropy copula problem
\begin{align}
\min_{\gamma \in \Gamma} ~~
\begin{cases}
\underset{P \in \mathcal{D}_d}{\max} ~~~ h\left( P \right)  \\
\text{s.t.} ~ E_P\left[ \phi_\gamma \left(\displaystyle \vu \right)  \right] = E_{P_{\bm{u}_{\bm{z}}}}\left[ \phi_\gamma \left(\displaystyle \vu \right)  \right]  \nonumber
\end{cases},
\end{align}
where $\gamma$ represents inner layers parameters. 

This time we need to estimate $E_{P_{\bm{u}_{\bm{z}}}}\left[ \phi_\gamma \left(\displaystyle \vu \right)  \right]$ from the data for as many inner layers parameters $\gamma$ as needed, which is far less data-efficient than MIND. Nonetheless, even this deductive twist to MIND would still be more data-efficient than MINE and NWJ in the primal space, as it would implicitly assume that we have enough data to fully characterize the copula distribution of $\left(\bm{y}, \bm{x}\right)$, but not its marginal distributions, whereas MINE and NWJ (in the primal space) require us to have enough data to be able to characterize the full joint distribution $P_{\bm{y}, \bm{x}}$ (i.e. all its marginals and its copula).

In the spirit of the LeanML design pattern, we stress that a surgical search for the best MIND statistics functions $\phi_\gamma$, such as by using gradient-descent, might not be necessary, and can be wasteful. Corollary 3.1 in \cite{pmlr-v130-kom-samo21a} provides that errors made estimating $h\left(\bm{u}_{\bm{y}} \right) + h\left(\bm{u}_{\bm{x}} \right)$ can cancel out errors made estimating $h\left( \bm{u}_{\bm{y}},   \bm{u}_{\bm{x}} \right)$, so that we may estimate the mutual information with high accuracy using MIND, even when some copula entropies weren't estimated as well. 

When data do not abound, we are better off choosing the statistics function $\phi$ so that i) $E_{P_{\bm{u}_{\bm{z}}}}\left[ \phi \left(\displaystyle \vu \right)  \right]$ reveals associations between coordinates of $\bm{z}$, and ii) $E_{P_{\bm{u}_{\bm{z}}}}\left[ \phi \left(\displaystyle \vu \right)  \right]$ can be estimated with a low enough variance using the amount of data available.

\section{Proofs}
\subsection{Proof of Proposition \ref{prop:ll}}
\label{proof:prop:ll}
Let $\mathcal{M}$ be a supervised leaner with predictive pmf or pdf $p_\mathcal{M}$. First, $\bar{\mathcal{LL}}\left( P_{\displaystyle \vy, \displaystyle \vx} \right) =\mathcal{LL}\left( \mathcal{M}^\infty \right)$. Second,
\begin{align*}
& \bar{\mathcal{LL}}\left( P_{\displaystyle \vy, \displaystyle \vx} \right)- \mathcal{LL}\left( \mathcal{M} \right) \\
&=  I\left(\displaystyle \vy ; \displaystyle \vx \right)-h\left(\displaystyle \vy\right)- E_{P_{\displaystyle \vy, \displaystyle \vx}} \left[ \log p_\mathcal{M}\left(\displaystyle \vy  \vert \displaystyle \vx  \right) \right] \\
&=  E_{P_{\displaystyle \vy, \displaystyle \vx}} \left[ \log p\left(\displaystyle \vy  \vert \displaystyle \vx  \right) \right]- E_{P_{\displaystyle \vy, \displaystyle \vx}} \left[ \log p_\mathcal{M}\left(\displaystyle \vy  \vert \displaystyle \vx  \right) \right] \\
&=  E_{P_{\displaystyle \vx }} \left[ E_{P_{\displaystyle \vy \vert \displaystyle \vx}} \left[  \log p\left(\displaystyle \vy  \vert \displaystyle \vx  \right) - \log p_\mathcal{M}\left(\displaystyle \vy  \vert \displaystyle \vx  \right) \right] \right] \\
&=E_{P_{\displaystyle \vx}} \left[\text{KL}\left( p\left(\displaystyle \vy  \vert \displaystyle \vx  \right) \vert \vert p_\mathcal{M}\left(\displaystyle \vy  \vert \displaystyle \vx  \right) \right)\right] \\
& \geq 0.
\end{align*}

\subsection{Proof of Proposition \ref{prop:best}}
\label{proof:prop:best}
We decompose the mutual information $I\left( \displaystyle \vy; \displaystyle \vx, \displaystyle \vz \right)$ in two different ways.
\begin{align*}
I\left( \displaystyle \vy; \displaystyle \vx, \displaystyle \vz \right) &= I\left( \displaystyle \vy; \displaystyle \vx \right) + I\left( \displaystyle \vy; \displaystyle \vz \vert \vx \right) \\
&= I\left( \displaystyle \vy; \displaystyle \vz \right) + I\left( \displaystyle \vy; \displaystyle \vx \vert \vz \right)
\end{align*}
Moreover, by definition of the generative graphical model of $\mathcal{M}$, $I\left( \displaystyle \vy; \displaystyle \vz \vert \vx \right) = 0$. Hence, by non-negativity of the mutual information, $I\left( \displaystyle \vy; \displaystyle \vx \right) \geq  I\left( \displaystyle \vy; \displaystyle \vz \right)$, and the equality holds if and only if $I\left( \displaystyle \vy; \displaystyle \vx \vert \vz \right)=0$. This condition is met  by $\mathcal{M}^\infty$.

The inequality $I\left( \displaystyle \vy; \displaystyle \vx \right) \geq  I\left( \displaystyle \vy; \displaystyle \vz \right)$ is known as the data processing inequality (\cite{incover1999elements}).
\begin{align*}
MSE \left(\mathcal{M}\right)  &\geq \mathbb{V}\text{ar} \left( y \right) e^{-2I\left( y  ; \displaystyle \vz \right)}  \\
& \geq \mathbb{V}\text{ar} \left( y \right) e^{-2I\left( y  ; \displaystyle \vx \right)} := \bar{MSE}\left( P_{y, \displaystyle \vx} \right),
\end{align*} where the second inequality stems from an application of the data processing inequality. 

Additionally, $P_{y \vert \displaystyle \vz^0} = P_{y}$ implies $I\left(y; \displaystyle \vz^0 \right) = 0$, and as $\mathcal{M}^0$ is unbiased, we get $MSE\left(\mathcal{M}^0\right) = \mathbb{V}\text{ar}\left( y \right)$. The data processing inequaliy is an equality for $\mathcal{M}^\infty$, and $\mathcal{M}^\infty$ is unbiased as $$E\left[y-z^\infty \right] = E\left[ E\left(y \vert \displaystyle \vx \right) - E\left(z^\infty  \vert \displaystyle \vx \right) \right] = 0.$$
Thus, the lowest MSE is reached by $\mathcal{M}^\infty$.

\subsection{Derivation of the extended strong Fano bound}
\label{sct:fano_deriv}
To prove Theorem \ref{theo:hacc}, we need a series of intermediary results. 

\textbf{Intuition:} Let us consider a classifier $\mathcal{M}$ with generative graphical model $y \rightarrow \displaystyle \vx  \rightarrow z $. As previously discussed, $z \in \mathcal{Y}$ typically represents the knowledge the model extracts about $y$ from  $\displaystyle \vx$. To simplify our illustration, we further restrict $z$ to be our prediction of $y$ after observing $\displaystyle \vx$, so that the accuracy of $\mathcal{M}$ reads $\mathbb{P}\left( y = z \right) := \mathcal{D}\left( \mathcal{M} \right)$. 

If we denote $\Pi$ the set of all (deterministic) permutation of $\{1, \dots, q\}$, then $$\mathbb{P}\left( y = z \right) \leq \underset{\pi \in \Pi}{\max} ~ \mathbb{P}\left( y = \pi(z) \right) := \mathcal{P}\left( \mathcal{M} \right).$$ Noting that
\begin{align*}
\mathbb{P}\left( y = \pi(z) \right) = \sum_{i=1}^q \mathbb{P}\left( y = \pi(i) \vert z = i \right) \mathbb{P}\left( z = i \right),
\end{align*}
and that 
\begin{align*}
\mathcal{P}\left( \mathcal{M} \right) & \leq  \sum_{i=1}^q \underbrace{\underset{\pi \in \Pi}{\max} ~ \mathbb{P}\left( y = \pi(i) \vert z = i \right)}_{\underset{j \in [1, q]}{\max} ~ \mathbb{P}\left( y = j \vert  z  = i \right)} \mathbb{P}\left( z = i \right) \\
& = E_{z} \left[ \underset{i \in [1, q]}{\max} ~ \mathbb{P}\left( y = i \vert  z \right) \right] :=  \mathcal{A}\left( \mathcal{M} \right),
\end{align*}
it follows that
\begin{align}
\label{eq:accuracies}
\mathcal{D}\left( \mathcal{M} \right) \leq  \mathcal{P}\left( \mathcal{M} \right)  \leq \mathcal{A}\left( \mathcal{M} \right).
\end{align}
These three quantities are very important to understand. $\mathcal{D}\left( \mathcal{M} \right)$ represents the probability of accurately predicting $y$ \textbf{\emph{as}} $z$. $\mathcal{P}\left( \mathcal{M} \right)$ represents the probability of accurately predicting $y$ \textbf{\emph{as a deterministic permutation or 1-to-1 function of}} $z$. $\mathcal{A}\left( \mathcal{M} \right)$ can be interpreted as the probability of accurately predicting $y$ \textbf{\emph{as any deterministic function of}} $z$, \textbf{\emph{1-to-1 or otherwise}}. Thus, even when the inequalities (\ref{eq:accuracies}) are strict, we may always find a deterministic function $f$ so that the classifier $\mathcal{M}_f$ with generative model $y \rightarrow \displaystyle \vx  \rightarrow z \rightarrow  f(z)$ satisfies $\mathcal{D}\left( \mathcal{M}_f \right) = \mathcal{A}\left( \mathcal{M} \right)$. To be specific,
$$f(z) = \argmax_{i \in \{1, \dots, q\}} \mathbb{P}\left(y=i \vert z \right).$$

As a result, determining the highest accuracy ($\mathcal{D}\left( \mathcal{M} \right)$) achievable by a classifier using $\displaystyle\vx$ to predict $y$ boils down to determining the highest possible value that $\mathcal{A}\left( \mathcal{M} \right)$ may take given the joint distribution $P_{y, \displaystyle\vx}$.

\begin{lemma}
\label{lem:1}
Let $y \sim P$ be a categorical random variable taking value in $\{1, \dots, q\}$, the $i$-th with probability $p_i$. The highest accuracy achievable when predicting $y$ solely from knowing $P$ is $$\mathcal{A}\left( P \right) := \underset{i \in [1, q]}{\max} ~ p_i,$$ and it is achieved by always predicting the most likely outcome.
\end{lemma}
\begin{proof}
A strategy predicting $y$ solely from knowing $P$ can be represented as a random variable $z$ with the same support as $y$ but that is independent from $y$, and with pmf $q_1, \dots, q_q$. Its accuracy is simply the probability that both variables are equal, 
\begin{align*}
\mathbb{P}\left(y = z \right) &= \mathbb{P}\left( \cup_{i=1}^q \left(y=i ~\&~ z=i\right) \right) \\
&= \sum_{i=1}^q p_iq_i \\
& \leq \left(  \underset{i \in [1, q]}{\max} ~ p_i \right) \sum_{i=1}^q q_i \\
&= \underset{i \in [1, q]}{\max} ~ p_i,
\end{align*}
with equality if and only if $q_j = 0$ for all $j \neq   \underset{i \in [1, q]}{\argmax} ~ p_i$.
\end{proof}

\begin{lemma}
\label{lem:2}
Among all discrete probability distributions on $\{1, \dots, q\}$ satisfying $\mathcal{A}\left( P \right) = a$, the one with the highest entropy is the one whose $(q-1)$ least likely outcomes have the same probability $\frac{1-a}{q-1}$, and it has Shannon entropy $$\bar{h}_q(a) := -a\log a - (1-a) \log \frac{1-a}{q-1}, ~~ a \geq \frac{1}{q}.$$
\end{lemma}
\begin{proof}
Let us  denote $\pi_1, \dots, \pi_q$ the probabilities of $P$ sorted in decreasing order and let us assume $a=\pi_1$. The Shannon entropy of $P$ reads
\begin{align*}
H(P) &= -\pi_1 \log \pi_1 - \sum_{i=2}^q \pi_i \log \pi_i \\
&= -\pi_1 \log \pi_1 + (1-\pi_1) \sum_{i=2}^q \frac{\pi_i}{1-\pi_1} \log \frac{1}{\pi_i} \\
& \leq -\pi_1 \log \pi_1 + (1-\pi_1) \log \sum_{i=2}^q \frac{\pi_i}{1-\pi_1} \frac{1}{\pi_i} \\
& = -\pi_1 \log \pi_1 + (1-\pi_1) \log \frac{q-1}{1-\pi_1} \\
& = -\pi_1 \log \pi_1 - (q-1) \frac{1-\pi_1}{q-1} \log \frac{1-\pi_1}{q-1}
\end{align*}
where the inequality is a direct application of Jensen's inequality to the strictly concave $\log$ function, and the equality holds if and only if $\pi_i$ are the same for $ i \geq 2$ and equal to $\frac{1-\pi_1}{q-1}$.
\end{proof}
\begin{figure}[h]
\includegraphics[width=0.45\textwidth]{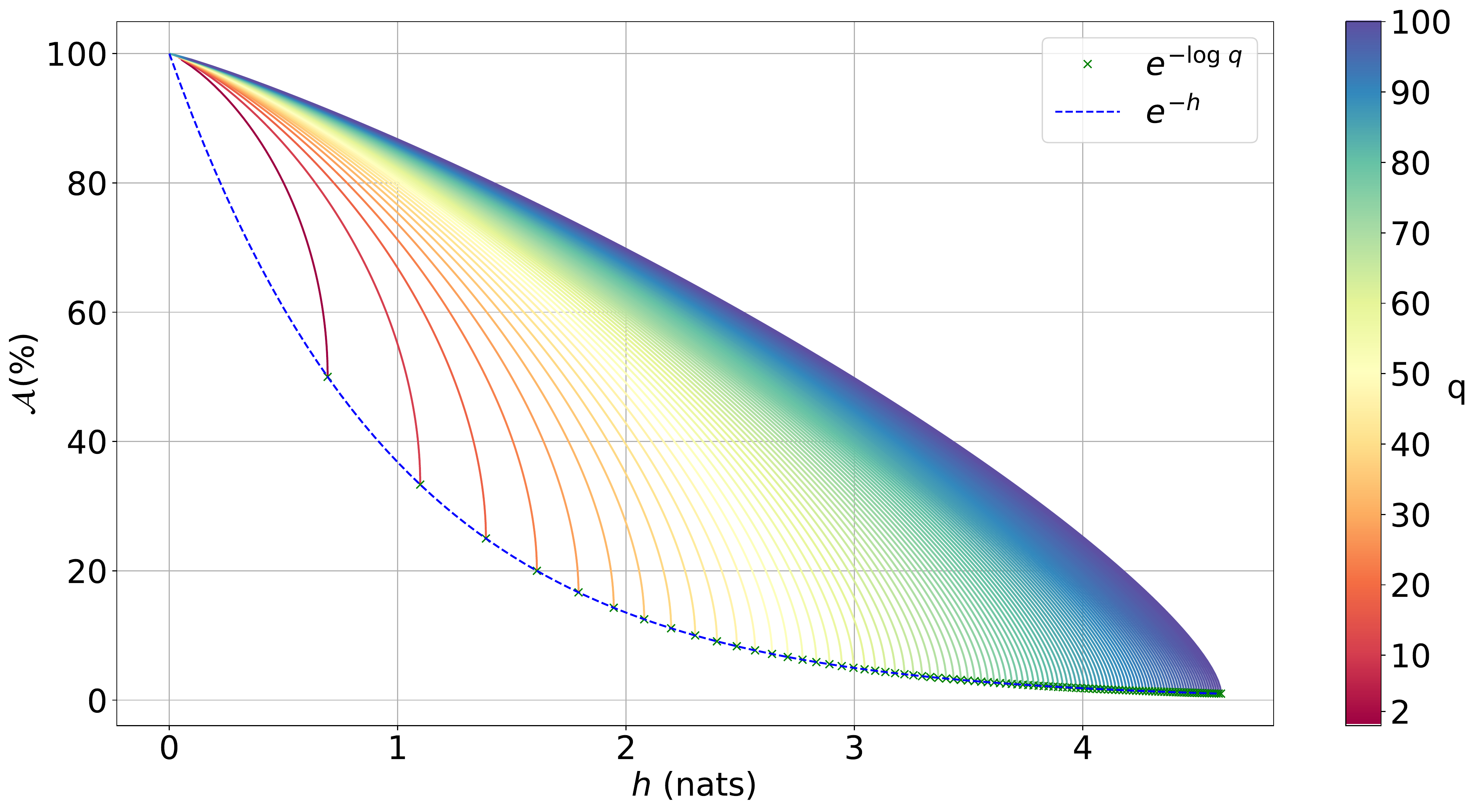}
\caption{Solid lines are plots of $h \to \bar{h}_q^{-1}(h)$ for various values of $q$.}
\label{fig:hqi}
\end{figure}

\begin{corollary}
\label{cor:1}
Among all discrete distributions taking $q$ distinct values and that have the same entropy, if there is one whose $(q-1)$ least likely outcomes have the same probability, then its highest outcome probability is the largest of them all.
\end{corollary}
\begin{proof}
Let $P$ and $Q$ be two distributions taking $q$ distinct values and that have the same entropy $H(P)=H(Q)$. Let's assume $P$'s $(q-1)$ least likely outcomes have the same probability. It follows from the previous lemma that the entropy of $Q$ is lower than the entropy of the distribution $Q^\prime$ having the same highest outcome probability $\mathcal{A}(Q)=  \mathcal{A}(Q^\prime)$ and whose $(q-1)$ least likely outcomes have the same probability (the lemma below justifies the existence of $P$ and $Q^\prime$): 
\begin{align*}
\bar{h}_q\left( \mathcal{A}(P)\right) 
= H(P) &= H(Q) \\
&< H(Q^\prime) =\bar{h}_q\left( \mathcal{A}(Q^\prime)\right) =\bar{h}_q\left( \mathcal{A}(Q)\right).
\end{align*} A study of the function $a \in [\frac{1}{q}, 1] \to \bar{h}_q(a)$ reveals that it is a decreasing function of $a$ for any $q$. Hence,  $\mathcal{A}(P) >  \mathcal{A}(Q)$.
\end{proof}

\begin{lemma}
\label{lem:3}
For any possible entropy value $h$ of a discrete distribution taking $q$ possible distinct values, there exists a discrete distribution whose entropy is $h$ and whose $(q-1)$ least likely outcomes have the same probability.
\end{lemma}
\begin{proof}
Let $h$ be the entropy of a discrete distribution on a set of size $q$. We have $h \in [0, \log q]$, as Shannon's entropy is non-negative, and the uniform distribution is maximum-entropy among all discrete distributions on a set with cardinality $q$ and it has entropy $\log q$. The probability $a$ of the most likely outcome ought to satisfy $a \geq 1/q$, otherwise all probabilities would sum to less than $1$. A study of the function $a \in [\frac{1}{q}, 1] \to \bar{h}_q(a)$ using the convention $0\log 0 = 0$ reveals that it is  differentiable, decreasing, concave, and invertible  on $[\frac{1}{q}, 1]$ and the image of $[\frac{1}{q}, 1]$ is $[0, \log q]$.
\end{proof}
\begin{theorem}
\label{theo:class_1}
Let $y$ be a categorical random variable taking up to $q$ distinct values, that has entropy $h$, but whose distribution we do not know. The highest accuracy we can achieve when predicting $y$ reads
\begin{align}
\mathcal{A}(h; q) := \bar{h}^{-1}_q(h),
\end{align}
where $h \to \bar{h}^{-1}_q(h)$  is the inverse of the function $a \in [\frac{1}{q}, 1] \to \bar{h}_q(a)$, as illustrated in Figure \ref{fig:hqi}.
\end{theorem}
\begin{proof}
This follows from Lemma \ref{lem:2}, Corollary \ref{cor:1} and Lemma \ref{lem:3}
\end{proof}

\begin{definition}
The accuracy of a classifier $\mathcal{M}$ predicting that the label $y$ associated to explanatory variables $\displaystyle\vx$ is $z$ is defined as
$$\mathcal{D}\left( \mathcal{M}\right) := \mathbb{P}(y=z).$$
\end{definition}

We may now state our main result expressing the highest accuracy achievable by a classifier as a function of the Shannon entropy of the label $y$ and the mutual information $I\left( y; \displaystyle \vx\right)$ between the label and explanatory variables. The idea behind the proof is to note that the highest possible $\mathcal{D}\left( \mathcal{M}\right)$ is the same as the highest possible $\mathcal{A}\left( \mathcal{M} \right) = E_{z} \left[ \mathcal{A}\left(P_{y \vert z} \right) \right]$ and to use previously established results to conclude.

\begin{theorem}
The highest accuracy $\bar{\mathcal{A}}(P_{y, \displaystyle \vx})$ achievable by a classifier using $\displaystyle \vx$ to predict a categorical random variable $y \in \{1, \dots, q\}$ satisfies the strong Fano inequality
\begin{align}
\label{eq:opt_class}
\bar{\mathcal{A}}(P_{y, \displaystyle \vx}) \leq \bar{h}^{-1}_q\left( h(y)  - I\left( y; \displaystyle \vx\right) \right).
\end{align}
Additionally,
\begin{align}
\label{eq:opt_class}
\bar{\mathcal{A}}(P_{y, \displaystyle \vx}) = \bar{h}^{-1}_q\left( h(y)  - I\left( y; \displaystyle \vx\right) \right).
\end{align}
and the \emph{oracle} classifier $\mathcal{M}^\infty$ achieves $\bar{\mathcal{A}}(P_{y, \displaystyle \vx})$ when the entropy of the conditional distribution, namely $h\left( y \vert \displaystyle \vx = * \right)$, is the same for all values $*$ of $\displaystyle \vx$ (i.e. $\displaystyle \vx$ is no more informative about $y$ in certain parts of the domain $\mathcal{X}$ than others), and when $q=2$ or the $(q-1)$ least likely outcomes under the conditional distribution $P_{y \vert \displaystyle \vx}$ are always equally likely (i.e. the information in $\displaystyle \vx$ about $y$ leaves no room for a clear runner-up).
\end{theorem}
\begin{proof}
As argued in the paper, the highest possible value achievable by $\mathcal{D}\left( \mathcal{M}\right)$ is the highest possible value achievable by $\mathcal{A}\left( \mathcal{M}\right)$, so that we may focus on the latter.

Let $\mathcal{M}$ be the classifier with generative graphical model $y \rightarrow \displaystyle \vx \rightarrow z$ and predictive distribution $P_{y \vert z}$.  
It follows from Corollary \ref{cor:1} that
$$\mathcal{A}\left(P_{y \vert z} \right) := \underset{i \in [1, q]}{\max} ~ \mathbb{P}\left( y = i \vert  z \right) \leq \bar{h}^{-1}_q\left( h(y \vert z=*) \right).$$

Taking the expectation  with respect to $z$, we get
\begin{align*}
\mathcal{A}\left( \mathcal{M} \right) &\leq  E_{z} \left[ \bar{h}^{-1}_q\left( h(y \vert z=*) \right) \right] \\
&\leq  \bar{h}^{-1}_q\left( E_{z} \left[  h(y \vert z=*) \right] \right) \\
&= \bar{h}^{-1}_q\left( h(y \vert z) \right),
\end{align*}
where the second inequality stems from the concavity of $\bar{h}^{-1}_q$. It follows from the data processing inequality, namely $$h(y) - h(y \vert \displaystyle \vx) = I\left( y; \displaystyle \vx \right)  \geq I\left(y; z \right) = h(y) - h(y \vert z),$$ that $h(y \vert z) \geq  h(y \vert \displaystyle \vx),$
which implies $\bar{h}^{-1}_q\left( h(y \vert z) \right) \leq \bar{h}^{-1}_q\left( h(y \vert \displaystyle \vx) \right)$ as $\bar{h}^{-1}_q$ is decreasing. Using $h(y \vert \displaystyle \vx)= h(y)  - I\left(y; \displaystyle \vx\right)$ we get $$\mathcal{A}\left( \mathcal{M} \right) \leq  \bar{h}^{-1}_q\left( h(y)  - I\left( y; \displaystyle \vx\right) \right).$$
By definition of $\mathcal{M}^\infty$, the data processing inequality is an equality when $\mathcal{M} = \mathcal{M}^\infty$. By strict concavity of $\bar{h}^{-1}_q$ the Jensen inequality we used is an equality for $\mathcal{M}^\infty$  if and only if $h(y \vert z^\infty=*)$, and therefore $h(y \vert \displaystyle \vx=*)$, is the same for every observed value of $z^\infty$ and $\displaystyle \vx$. As for the application of Corollary \ref{cor:1}, the inequality is an equality when the $(q-1)$ least likely outcomes of $P_{y \vert z^\infty} = P_{y \vert \displaystyle \vx}$  are equally probable.
\end{proof}

\begin{table*}[!ht]
\centering
\resizebox{\textwidth}{!}{
\begin{tabular}{llrrllll}
\toprule
                       Dataset &   Problem Type &      n &   d & Number of Classes &     RMSE & R-Squared & Classification Accuracy \\
\midrule
       Skin Segmentation (UCI) & classification & 245057 &   3 &                 2 &        - &      0.64 &                    1.00 \\
               Bank Note (UCI) & classification &   1372 &   4 &                 2 &        - &      0.75 &                    1.00 \\
        Water Quality (Kaggle) & classification &   3276 &   9 &                 2 &        - &      0.04 &                    0.65 \\
                 Shuttle (UCI) & classification &  58000 &   9 &                 7 &        - &      0.74 &                    1.00 \\
             Magic Gamma (UCI) & classification &  19020 &  10 &                 2 &        - &      0.67 &                    0.98 \\
                   Avila (UCI) & classification &  20867 &  10 &                12 &        - &      0.97 &                    1.00 \\
              Titanic (Kaggle) & classification &    891 &  11 &                 2 &        - &      0.48 &                    0.89 \\
         Heart Attack (Kaggle) & classification &    303 &  13 &                 2 &        - &      0.64 &                    0.95 \\
        Heart Disease (Kaggle) & classification &    303 &  13 &                 2 &        - &      0.64 &                    0.95 \\
           EEG Eye State (UCI) & classification &  14980 &  14 &                 2 &        - &      0.75 &                    1.00 \\
                   Adult (UCI) & classification &  48843 &  14 &                 3 &        - &      0.67 &                    1.00 \\
      Letter Recognition (UCI) & classification &  20000 &  16 &                26 &        - &      1.00 &                    1.00 \\
    Diabetic Retinopathy (UCI) & classification &   1151 &  19 &                 2 &        - &      0.53 &                    0.90 \\
          Bank Marketing (UCI) & classification &  41188 &  20 &                 2 &        - &      0.72 &                    1.00 \\
            Card Default (UCI) & classification &  30000 &  23 &                 2 &        - &      0.75 &                    1.00 \\
                 Landsat (UCI) & classification &   6435 &  36 &                 6 &        - &      0.97 &                    1.00 \\
       Sensor Less Drive (UCI) & classification &  58509 &  48 &                11 &        - &      0.99 &                    1.00 \\
             APS Failure (UCI) & classification &  76000 & 170 &                 2 &        - &      0.75 &                    1.00 \\
             Power Plant (UCI) &     regression &   9568 &   4 &                 - &     4.31 &      0.94 &                       - \\
                Air Foil (UCI) &     regression &   1503 &   5 &                 - &     3.09 &      0.80 &                       - \\
     Yacht Hydrodynamics (UCI) &     regression &    308 &   6 &                 - &     1.46 &      0.99 &                       - \\
             Real Estate (UCI) &     regression &    414 &   6 &                 - &     6.05 &      0.80 &                       - \\
       Energy Efficiency (UCI) &     regression &    768 &   8 &                 - &     1.35 &      0.98 &                       - \\
                Concrete (UCI) &     regression &   1030 &   8 &                 - & 5.35e-01 &      1.00 &                       - \\
                 Abalone (UCI) &     regression &   4177 &   8 &                 - & 2.50e-02 &      1.00 &                       - \\
      White Wine Quality (UCI) &     regression &   4898 &  11 &                 - & 7.27e-01 &      0.33 &                       - \\
             Air Quality (UCI) &     regression &   8991 &  14 &                 - & 1.14e-03 &      1.00 &                       - \\
        Naval Propulsion (UCI) &     regression &  11934 &  16 &                 - & 1.47e-02 &      0.01 &                       - \\
            Bike Sharing (UCI) &     regression &  17379 &  18 &                 - &     3.98 &      1.00 &                       - \\
               Parkinson (UCI) &     regression &   5875 &  20 &                 - & 7.60e-13 &      1.00 &                       - \\
       Facebook Comments (UCI) &     regression & 209074 &  53 &                 - &     6.46 &      0.98 &                       - \\
             Online News (UCI) &     regression &  39644 &  58 &                 - &      219 &      1.00 &                       - \\
       Social Media Buzz (UCI) &     regression & 583250 &  77 &                 - &     5.14 &      1.00 &                       - \\
House Prices Advanced (Kaggle) &     regression &   1460 &  79 &                 - &    1,747 &      1.00 &                       - \\
       Superconductivity (UCI) &     regression &  21263 &  81 &                 - & 2.72e-05 &      1.00 &                       - \\
     Year Prediction MSD (UCI) &     regression & 515345 &  90 &                 - &     1.94 &      0.97 &                       - \\
           Blog Feedback (UCI) &     regression &  60021 & 280 &                 - &     1.01 &      0.89 &                       - \\
               CT Slices (UCI) &     regression &  53500 & 385 &                 - & 3.61e-01 &      1.00 &                       - \\
\bottomrule
\end{tabular}}
\caption{Highest performances achievable in 38 of the most popular UCI and Kaggle classification and regression problems.  In the Kaggle House Prices Advanced Regression experiment,  we used as target house prices, not their logarithms. }
\label{tab:exhaustive}
\end{table*}

\end{document}